\documentclass{article}

    \usepackage[final]{neurips_2020}

     \usepackage{neurips_2020}

\usepackage[utf8]{inputenc} % allow utf-8 input
\usepackage[T1]{fontenc}    % use 8-bit T1 fonts
\usepackage{hyperref}       % hyperlinks
\usepackage{url}            % simple URL typesetting
\usepackage{booktabs}       % professional-quality tables
\usepackage{amsfonts}       % blackboard math symbols
\usepackage{nicefrac}       % compact symbols for 1/2, etc.
\usepackage{microtype}      % microtypography
\usepackage{enumitem}
\usepackage{booktabs}
\usepackage{adjustbox}

\usepackage{graphicx}
\usepackage{url}
\usepackage{multirow}
\usepackage{subcaption}
\usepackage{amsthm}
\usepackage{amsmath}
\usepackage{mathtools}
\usepackage{float}
\usepackage[ruled, vlined]{algorithm2e}
\usepackage[noend]{algorithmic}

\usepackage[english]{babel}
\usepackage{amsmath,amsfonts,bm}

\def\eqref#1{equation~(\ref{#1})}
\def\1{\bm{1}}

\def\rd{{\textnormal{d}}}

\def\rt{{\textnormal{t}}}

\def\rx{{\textnormal{x}}}
\def\ry{{\textnormal{y}}}

\def\vmu{{\bm{\mu}}}
\def\vtheta{{\bm{\theta}}}

\def\vb{{\bm{b}}}

\def\vd{{\bm{d}}}

\def\vf{{\bm{f}}}
\def\vg{{\bm{g}}}

\def\vl{{\bm{l}}}

\def\vr{{\bm{r}}}

\def\vx{{\bm{x}}}
\def\vy{{\bm{y}}}

\def\mC{{\bm{C}}}
\def\mD{{\bm{D}}}

\def\mI{{\bm{I}}}

\def\mP{{\bm{P}}}

\def\mR{{\bm{R}}}

\def\mW{{\bm{W}}}
\def\mX{{\bm{X}}}
\def\mY{{\bm{Y}}}
\def\mZ{{\bm{Z}}}

\def\mSigma{{\bm{\Sigma}}}

\DeclareMathAlphabet{\mathsfit}{\encodingdefault}{\sfdefault}{m}{sl}
\SetMathAlphabet{\mathsfit}{bold}{\encodingdefault}{\sfdefault}{bx}{n}
\newcommand{\tens}[1]{\bm{\mathsfit{#1}}}

\def\tX{{\tens{X}}}

\newcommand{\E}{\mathbb{E}}

\newcommand{\R}{\mathbb{R}}

\DeclareMathOperator*{\argmax}{arg\,max}
\DeclareMathOperator*{\argmin}{arg\,min}

\usepackage{color}

\newtheorem{theorem}{Theorem}[section]
\newtheorem{lemma}[theorem]{Lemma}

\title{Optimizing Mode Connectivity via Neuron Alignment}

\author{%
  N. Joseph Tatro
    \\
  Dept. of Mathematical Sciences\\
  Rensselaer Polytechnic Institute \\
  Troy, NY \\
  \texttt{tatron@rpi.edu} \\
   \And
   Pin-Yu Chen \\
   IBM Research \\
   Yorktown Heights, NY \\
   \texttt{pin-yu.chen@ibm.com} \\
   \And
   Payel Das \\
   IBM Research \\
   Yorktown Heights, NY \\
   \texttt{daspa@us.ibm.com} \\
   \And
   Igor Melnyk \\
   IBM Research \\
   Yorktown Heights, NY \\
   \texttt{igor.melnyk@ibm.com} \\
   \And
   Prasanna Sattigeri \\
   IBM Research \\
   Yorktown Heights, NY \\
   \texttt{psattig@us.ibm.com} \\
   \And
   Rongjie Lai \\
   Dept. of Mathematical Sciences \\
   Rensselaer Polytechnic Institute \\
   Troy, NY \\
   \texttt{lair@rpi.edu} \\
}

\begin{document}

\maketitle

\begin{abstract}
%   The loss landscapes of deep neural networks are not well understood due to their high nonconvexity. Empirically, the local minima of these loss functions can be connected by a learned curve in model space, along which the loss remains nearly constant; a feature known as mode connectivity. Yet, current path finding algorithms do not consider the influence of symmetry in the loss surface created by model weight permutations. We propose a framework to investigate the effect of symmetry on landscape connectivity by directly optimizing the weight permutations of the networks being connected. To learn a locally optimal permutation, we introduce both a proximal alternating minimization scheme with some convergence guarantees as well as an inexpensive heuristic referred to as neuron alignment. Empirically, optimizing the weight permutation is critical for efficiently learning a simple, planar, low-loss curve between networks that successfully generalizes. Surprisingly, our alignment method can significantly alleviate the recently identified robust loss barrier on the path connecting two adversarial robust models and find more robust and accurate models on the path.
The loss landscapes of deep neural networks are not well understood due to their high nonconvexity. Empirically, the local minima of these loss functions can be connected by a learned curve in model space, along which the loss remains nearly constant; a feature known as mode connectivity. Yet, current curve finding algorithms do not consider the influence of symmetry in the loss surface created by model weight permutations. We propose a more general framework to investigate the effect of symmetry on landscape connectivity by accounting for the weight permutations of the networks being connected. To approximate the optimal permutation, we introduce an inexpensive heuristic referred to as neuron alignment. Neuron alignment promotes similarity between the distribution of intermediate activations of models along the curve. We provide theoretical analysis establishing the benefit of alignment to mode connectivity based on this simple heuristic. We empirically verify that the permutation given by alignment is locally optimal via a proximal alternating minimization scheme. Empirically, optimizing the weight permutation is critical for efficiently learning a simple, planar, low-loss curve between networks that successfully generalizes. Our alignment method can significantly alleviate the recently identified robust loss barrier on the path connecting two adversarial robust models and find more robust and accurate models on the path. Code is available at \url{https://github.com/IBM/NeuronAlignment}.
% \footnote{Code is available at \url{https://github.com/IBM/NeuronAlignment}}

\end{abstract}
\section{Introduction}
\label{sec:intro}

Loss surfaces of neural networks have been of recent interest in the deep learning community both from a numerical \citep{dauphin2014identifying, sagun2014explorations} and a theoretical \citep{choromanska2014loss, safran2015quality} perspective. Their optimization yields interesting examples of a high-dimensional non-convex problem, where counter-intuitively gradient descent methods successfully converge to non-spurious minima. Practically, recent advancements in several applications have used insights on loss surfaces to justify their approaches. For instance, \citet{moosavi2019robustness} investigates regularizing the curvature of the loss surface to increase the robustness of trained models. 

One  interesting question about these non-convex loss surfaces is to what extent  trained models, which correspond to local minima, are connected. Here, \textit{connection} denotes the existence of a path between the models, parameterized by their weights, along which loss is nearly constant. There has been conjecture that such models are connected asymptotically, with respect to the width of hidden layers. Recently, \citet{freeman2016topology} proved this for rectified networks with one hidden layer. 

% \PY{Moreover, \citet{zhao2020bridging} uses the property of mode connectivity to recover adversarially tampered models with limited clean data.}

When considering the connection between two neural networks, it is important for us to consider what properties of the neural networks are intrinsic. Intuitively, there is a permutation ambiguity in the indexing of units in a given hidden layer of a neural network, and as a result, this ambiguity extends to the network weights themselves. Thus, there are numerous equivalent points in model space that correspond to a given neural network,  creating weight symmetry in the loss landscape. It is possible that the minimal loss paths between a network and all networks equivalent to a second network could be quite different. If we do not consider the best path among this set, we could fail to see to what extent models are intrinsically connected. Therefore, in this work we propose to develop a technique for more consistent model interpolation / optimal connection finding by investigating the effect of weight symmetry in the loss landscape.% in an effort to find more optimal curves. 
The analyses and results will give us insight into the geometry of deep neural network loss surfaces that is often hard to study theoretically. 
% The study of curve finding significant in its own right, as it is a recent tool in the study of topics such as the lottery ticket hypothesis \citep{anonymous2020mode}. 

% \subsection{Related Work}
\paragraph{Related Work}
\citet{freeman2016topology} were one of the first to rigorously prove that one hidden layer rectified networks are asymptotically connected. Recent numerical works have demonstrated learning parameterized curves along which loss is nearly constant. Concurrently, \citet{garipov2018loss} learned Bezier curves while \citet{draxler2018essentially} learned a curve using nudged elastic band energy. 
% between two models. 
\citet{gotmare2018using} showed that these algorithms work for models trained using different hyperparameters. Recently, \citet{kuditipudi2019explaining} analyzed the connectivity between $\epsilon$-dropout stable networks. This body of work can be seen as the extension of the linear averaging of models studied in \citep{goodfellow2014qualitatively}. Recent applications of alignment include model averaging in \citep{singh2019model} and federated learning in \citep{wang2018understanding}. Recently, \citet{zhao2020bridging} used mode connectivity to recover adversarially tampered models with limited clean data.

The symmetry groups in neural network weight space have long been formally studied \citep{chen1993geometry}. Ambiguity due to scaling in the weights has received much attention. Numerous regularization approaches based on weight scaling such as in \citep{cho2017riemannian} have been proposed to improve the performance of learned models. Recently, \citet{brea2019weight} studied the existence of \textit{permutation plateaus} in which the neurons in the layer of a network can all be permuted at the same cost. 
% Developed independently of our work, \citet{anokhin2020lowloss} recently considered the influence of weight symmetry in mode connectivity, while learning curves to connect unpermuted models.    

A second line of work studies network similarity. \citet{kornblith2019similarity} gives a comprehensive review while introducing centered kernel alignment (CKA) for comparing different networks. CKA is an improvement over the CCA technique introduced in \citep{raghu2017svcca} and explored further in \citep{morcos2018insights}. Another contribution is the neuron alignment algorithm of \citep{li2016convergent}, which empirically showed that two networks of same architecture learn a subset of similar features.

% \subsection{Contributions}
\paragraph{Contributions}
We summarize our main contributions as follows:
%  \\ 1. We generalize the problem of learning a curve between two neural networks, optimizing both the permutation of the second model weights and the curve parameters. Additionally, we introduce neuron alignment to learn an approximation to the optimal permutation for \textit{aligned} curves.
% % \PY{(I think this one needs to be further polished, and perhaps don't mention Li work so not to undermine our novelty)}
% \\
% 2. We empirically demonstrate that alignment promotes similarity of the intermediate activation distributions of models along the learned curve.
% % with those of the endpoint models. 
% Additionally, we offer analysis for a specific alignment algorithm, establishing a mechanism through which alignment improves mode connectivity. \\
% 3. We perform experiments on 3 datasets and 3 architectures affirming that more optimal curves can be learned faster with neuron alignment. Through utilizing a rigorous optimization  method, Proximal Alternating Minimization (PAM), we observe that this aligned permutation is close to being locally optimal and consistently outperforms solutions given from random permutation initialization. \\
% 4. For learned curves connecting adversarial robust models, we observe that the robust loss barrier recently identified in \citep{zhao2020bridging} can be greatly reduced with alignment, allowing us to find more accurate robust models on the path.
\begin{enumerate}[leftmargin=1.0cm]
    \item We generalize learning a curve between two neural networks by optimizing both the permutation of the second model weights and the curve parameters. Additionally, we introduce neuron alignment to learn an approximation to the optimal permutation for \textit{aligned} curves.
% \PY{(I think this one needs to be further polished, and perhaps don't mention Li work so not to undermine our novelty)}
\item We demonstrate that alignment promotes similarity of the intermediate activations of models along the learned curve.
% with those of the endpoint models. 
Additionally, we offer analysis for a specific alignment algorithm, establishing a mechanism through which alignment improves mode connectivity. 
\item We perform experiments on 3 datasets and 3 architectures affirming that more optimal curves can be learned with neuron alignment. Through utilizing a rigorous optimization  method, Proximal Alternating Minimization, we observe that this aligned permutation is nearly locally optimal and consistently outperforms solutions given from random initialization. 
\item  For learned curves connecting adversarial robust models, we observe that the robust loss barrier recently identified in \citep{zhao2020bridging} can be greatly reduced with alignment, allowing us to find more accurate robust models on the path.
\end{enumerate}

\section{Background}
\label{sec:background}

\begin{figure}[tb]
    \begin{center}
    % \begin{subfigure}[b]{0.48\linewidth}
    % 	\centering
	   % \includegraphics[width=\linewidth]{figures/fin_align_layer.png}
	   %% \caption{Cross-correlation between neurons}
	   %% \label{subfig:corr}
    % \end{subfigure}
    % \hfill
    \begin{subfigure}[b]{0.32\linewidth}
    	\centering
	    \includegraphics[width=\linewidth]{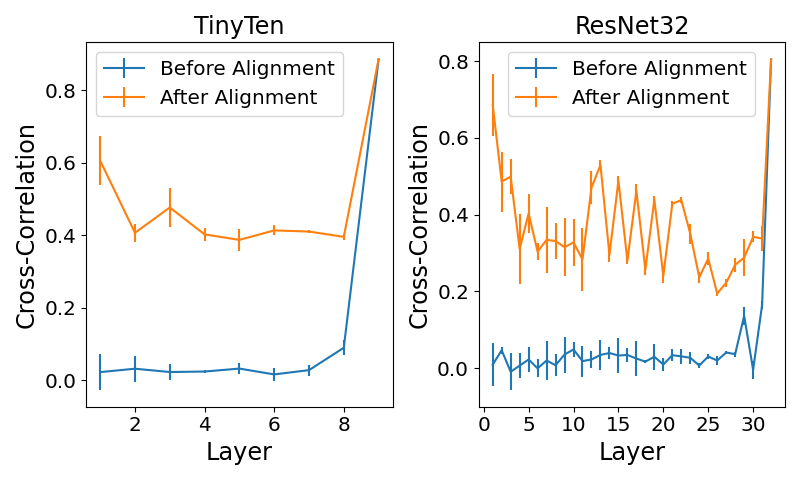}
	    \caption{Correlation between curve endpoints}
	   % \caption{Mean cross-correlation at each layer}
	    \label{subfig:corr_endpoint}
    \end{subfigure}
    \hfill
    \begin{subfigure}[b]{0.32\linewidth}
    	\centering
	    \includegraphics[width=\linewidth]{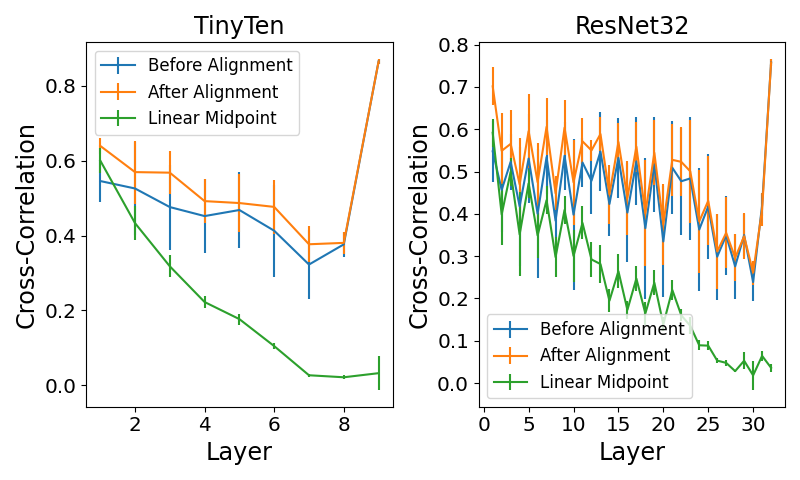}
	    \caption{Correlation between endpoint and unaligned curve midpoint}
	   % \caption{Mean cross-correlation at each layer}
	    \label{subfig:corr_end_unalignmid}
    \end{subfigure}
    \hfill 
    \begin{subfigure}[b]{0.32\linewidth}
    	\centering
	    \includegraphics[width=\linewidth]{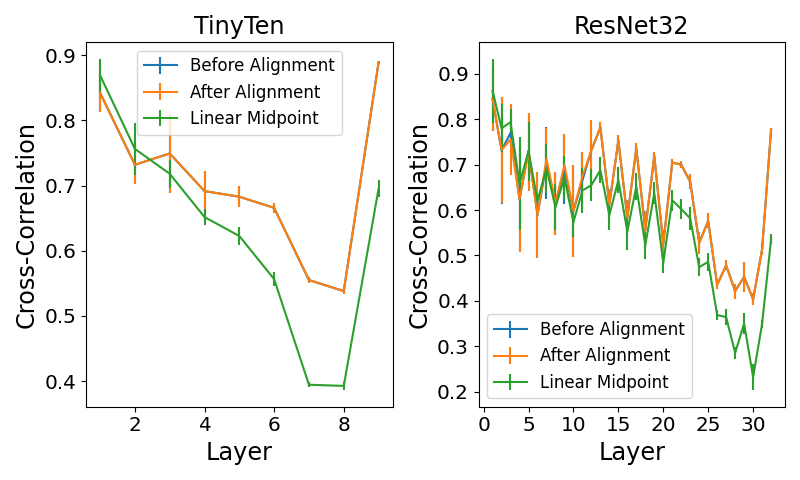}
	    \caption{Correlation between endpoint and aligned curve midpoint}
	   % \caption{Mean cross-correlation at each layer}
	    \label{subfig:corr_end_alignmid}
    \end{subfigure}
    \end{center}
    % \vspace{-2mm}
    % \textbf{(\ref{subfig:corr})}
    % (\ref{subfig:viz_align})
    \caption{\textbf{Left:} The mean cross-correlation between corresponding units of two models for each layer before and after alignment. The standard deviation of this correlation signature over a set of different network pairs is displayed. The quality of correspondence at each layer is improved by alignment. \textbf{Middle/Right:} These plots display the correlation signatures for the endpoint models and the midpoint of the learned \textbf{unaligned}/\textbf{aligned} curves. This signature is shown before and after calculating the alignment of the midpoint to the endpoint model. For comparison, the signature between the endpoint and the linear initialization midpoint is shown. For the aligned curve, the orange and blue curves are almost identical. We empirically see that alignment promotes greater similarity in the intermediate activations of the midpoint model to those of the endpoint model.}
    \label{fig:corr_align}
    \vspace{-2mm}
\end{figure}

% In this section we review the existing approaches for loss optima connectivity and neuron alignment. 

\subsection{Mode Connectivity}
% \label{subsec:lossconnect}
To learn the minimal loss path connecting two neural networks, $\vtheta_1$ and $\vtheta_2$, with $N$ parameters, we utilize the approach of \citep{garipov2018loss}. We search for the connecting path, $\vr : [0,1] \mapsto \R^{N}$, that minimizes the average of the loss, $\mathcal{L}$, along the path. This problem is formalized in \eqref{model:curve_finding}.    
\begin{equation} \label{model:curve_finding}
\begin{split}
    \vr^* = & \arg \min_{\vr} \quad  \frac{\int_{t \in [0, 1]} \mathcal{L}(\vr(t)) \|\vr'(t)\|  dt}{\int_{t \in [0, 1]}\|\vr'(t)\|  dt}  \qquad \text{subject to} \quad  \vr(0) = \vtheta_1, \vr(1) = \vtheta_2.
    % \vr^* = & \arg \min_{\vr} \quad  \int_{t \in [0, 1]} \mathcal{L}(\vr(t)) \|\vr'(t)\|  dt /  \text{arclength}(\vr),  \quad \text{subject to} \quad  \vr(0) = \vtheta_1, \vr(1) = \vtheta_2.
\end{split}
\end{equation}
For tractability, $\vr^*$ can be approximated by a parameterized curve $\vr_\phi$, where $\phi$ denotes the parameters. 
% For instance, as described in Section \ref{sec:experiments}, this paper will be using the quadratic Bezier curve. 
An arclength parameterization, $||r'(t)|| = 1$ for all $t$, is assumed to make optimization computationally feasible. If the endpoints are global minima and a flat loss path exists, then the optimal objective of \eqref{model:curve_finding} is unchanged. Algorithm \ref{alg:curve_find} in Appendix \ref{sec:algs} addresses computationally solving this problem. For clarity, $r_\phi$ is the curve on the loss surface, while $\vr_\phi(t)$ is a network on the curve. 

\subsection{Neuron Alignment via Assignment}
\label{subsec:background_align}

Neuron alignment refers to techniques of determining a bipartite match between neurons in the same layer of different neural networks. Given input $\rd$ drawn from the input data distribution $D$, let $\mX_{l, i}^{(1)}(\rd) \in \R^{k}$ represent the activations of channel $i$ in layer $l$ of network $\vtheta_1$, where $k$ is the number of units in the channel. A channel could be a unit in a hidden state or a filter in a convolutional layer.
% , where $k$ would be $1$ or the number of pixels in the filter respectively. 

Given two networks of the same architecture, $\vtheta_1$ and $\vtheta_2$, we define a permutation, $\mP_l: K_l \rightarrow K_l$, that maps from the index set of neurons in layer $l$ of $\vtheta_1$ to $\vtheta_2$. This provides a correspondence between the neurons, and we can associate a cost function $c: \mathbb{R}^k \times \mathbb{R}^k \rightarrow \mathbb{R}^{+}$ for the individual correspondences to minimize. This is exactly solving the \textit{assignment problem} \citep{burkard1999linear},
\begin{equation}
\label{eq:opt_trans}
    \mP_l^* = \argmin_{\mP_l \in \Pi_{K_l}} \sum_{i \in K_l} c(\mX_{l, i}^{(1)}, \mX_{l, \mP_l(i)}^{(2)}). 
\end{equation}
% Alternatively stated in the language of optimal transport, $\mP_l$ is essentially a transport plan and $c$ is the corresponding cost function \citep{peyre2019computational}. Then we are looking for the transportation plan $\mP_l$ that minimizes the total cost with additional constraints related to bipartite correspondence.
This has natural ties to optimal transport \citep{peyre2019computational} discussed further in Appendix \ref{sec:theory_neuron_alignment}. There are many valid ways to construct $\mP_l^*$ using different cost functions. For instance, in \citep{li2016convergent}, the cost function $c$ corresponds to $c(\vx, \vy) := 1 - \text{correlation}(\vx, \vy)$ for $\vx, \vy \in \mathbb{R}^k$. 
For clarity, $\text{correlation}(\vx, \vy)$ between channels is defined by the following equation, with channel-wise means $\vmu$ and standard deviations $\mSigma$, 
\begin{equation}
    \text{correlation}(\vx, \vy) = \frac{\vx - \vmu_x}{\mSigma_x}^T \frac{\vy - \vmu_y}{ \mSigma_y}.
\end{equation}
% \begin{equation}
%     \label{eq:li_na}
%     c(\vx, \vy) := 1 - \text{correlation}(\vx, \vy) \quad \text{for} \quad \vx, \vy \in \mathbb{R}^k.
% \end{equation}
We also note that while the method of \citep{li2016convergent} uses post-activations, alignment can also be done using pre-activations. That is, $\mX_{l, i}^{(1)}$ corresponds to the values at the neurons before the application of a pointwise nonlinearity, $\sigma$. Certain choices of the cost function $c$ and definition of the activations allow for more tractable theoretical analysis of alignment. Regardless, all reasonable choices promote a kind of similarity in the intermediate activations between two networks.   

The correlation-based alignment is visualized in Figure \ref{subfig:corr_endpoint}. This plot displays mean cross-correlation at each layer between corresponding neurons. With this \textit{correlation signature} being increased highly with alignment, we are matching a subset of highly correlated features. 
%\vspace{-3mm}

\subsection{Adversarial Training}
\label{subsec:background_adv_training}
A recent topic of interest has been learning robust models that can withstand adversarial attacks. We specifically consider Projected Gradient Descent (PGD) attacks as described in \citep{aleks2017deep}. This is an evasion attack that adds optimized $l_{\infty}$ bounded noise to input to degrade accuracy. Security to these are important as they can be conducted without access to model parameters as in 
\citep{Chen_2017}. Moreover, adversarial attacks can be used during model training to improve adversarial 
robustness, a method known as adversarial training \citep{goodfellow2014explaining,aleks2017deep}.

\section{Mode Connectivity with Weight Symmetry}
\label{sec:connect_up_to_symmetry}

We clarify the idea of weight symmetry in a neural network. Let $\vtheta_1$ be a network on the loss surface parameterized by its weights. A permutation $\mP_l$ is in $\Pi_{|K_l|}$, the set of permutations on the index set of channels in layer $l$. For simplicity suppose we have an $L$ layer feed-forward network with pointwise activation function $\sigma$, weights $\{W_l\}_{l=1}^L$, and input $X_0$. Then the weight permutation ambiguity becomes clear when we introduce the following set of permutations to the feedforward equation:
\begin{equation}\label{eq:perm_feedforward}
\begin{split}
    \mY := & \mW_L \mP_{L-1}^T \circ \sigma \circ \mP_{L-1} \mW_{L-1} \mP_{L-2}^T \circ \ldots \circ \sigma \circ  \mP_1 \mW_1 \mX_0.  
\end{split}
\end{equation}
% \begin{equation}\label{eq:perm_feedforward}
% \begin{split}
% \mY := & \mW_L \mP_{L-1}^T \mA_{L-1} \\
% & \quad \mA_k = \sigma \circ \mP_{k} \mW_{k} \mP_{k-1}^T \mA_{k-1}, \mA_1 = \sigma \circ \mP_{1} \mW_{1} \mX_0 
% \end{split}
% \end{equation}
Then we can define the network weight permutation $\mP$ as the block diagonal matrix, $\text{blockdiag}( \mP_1, \mP_2, ..., \mP_{L-1})$. Additionally, $\mP \vtheta$ denotes the network parameterized by the weights $[\mP_1 \mW_1, \mP_2 \mW_2 \mP_1^T, ..., \mW_L \mP^T_{L-1} ]$. Note that we omit permutations $\mP_0$ and $\mP_L$, as the input and output channels of neural networks have a fixed ordering. It is critical that $\mP_i$ is a permutation, as more general linear transformations do not commute with $\sigma$. 
% , so they correspond to the identity $\mI$. 
Without much difficulty this framework generalizes for more complicated architectures. We discuss this for residual networks in Appendix \ref{sec:resnet_align}. 

% \subsection{Curve Finding up to Symmetry}
% \label{subsec:joint_model}

From \eqref{eq:perm_feedforward}, it becomes clear that the networks $\vtheta_1$ and $\mP \vtheta_1$ share the same structure and intermediate outputs up to indexing. Taking weight symmetry into account, we can find the optimal curve connecting two networks up to symmetry with the model in \eqref{eq:curve_find_sym}.
% \begin{equation} \label{eq:curve_find_sym}
% \begin{split}
%     \min\limits_{\phi, \mP}  \quad & \E_{\rt \sim U} [\mathcal{L}(\vr_\phi(t))] \\
%     \text{subject to} \quad  & \vr_\phi(0) = \vtheta_1, \vr_\phi(1) = \mP \vtheta_2,  \\ 
%     & \mP = \text{blockdiag}(\mP_1, \mP_2, ..., \mP_{L-1}) \quad \text{where} \quad  
%     \mP_l \in \Pi_{|K_l|}   
% \end{split}
% \end{equation}
\begin{equation} \label{eq:curve_find_sym}
\begin{split}
    % \min\limits_{\phi, \mP}  \quad & \E_{\rt \sim U} [\mathcal{L}(\vr_\phi(t))] \\
    % \text{subject to} \quad  & \vr_\phi(0) = \vtheta_1, \vr_\phi(1) = \mP \vtheta_2, \mP = \text{blockdiag}(\{\mP_i\}_{i=1}^{L-1}) \quad \text{where} \quad  
    % \mP_l \in \Pi_{|K_l|} 
    \min\limits_{\phi, \mP}  \quad & \E_{\rt} [\mathcal{L}(\vr_\phi(t))] \quad s.t. \quad  \vr_\phi(0) = \vtheta_1, \vr_\phi(1) = \mP \vtheta_2, \mP = \text{blockdiag}(\{\mP_i \in \Pi_{|K_i|}\}_{i=1}^{L-1}). 
\end{split}
\end{equation}

\subsection{Neuron Alignment for Approximating the Optimal Permutation}
\label{subsec:na_heuristic}

% This work focuses on using neuron alignment to approximate the optimal permutation, $\mP^*$ in \eqref{eq:curve_find_sym}. 
% % We begin by motivating the use of alignment algorithms. 
% In a simplified example, given two image classifiers, imagine there exists a \textit{cat filter} and a \textit{dog filter} in a given layer for each network. Intuitively, the curve learned by mode connectivity should connect the weights between the corresponding \textit{cat filters} and \textit{dog filters}, instead of mixing them. Assuming that the activation distributions of the \textit{cat filters} of the two networks are more similar than to those of the other \textit{dog filter}, this can be achieved by connecting aligned networks.
% As stated in \ref{subsec:background_align}, this is because alignment puts neurons in correspondence that are similar under some metric. 
% We refer once again to Figure \ref{fig:corr_align} to stress that this is empirically reasonable. 

% \subsubsection{Alignment Promotes Similarity of Models on Curve to Endpoints}
With the correlation signatures of Figure \ref{fig:corr_align}, it is clear that trained networks of the same architecture share some similar structure. Then we could expect that if a model along the curve, $\vr(t)$, is optimal, it shares some structure of $\vr(0)$ and $\vr(1)$. To this end, it is sensible that the curve connecting the models interpolates weights of similar filters. This motivates our use of neuron alignment to estimate $\mP^*$ in \eqref{eq:curve_find_sym}, so that we are learning a curve between \textit{aligned networks}. In Section \ref{sec:experiments}, we confirm increased similarity of intermediate activations of models along the curve to endpoint models. 

\paragraph{Theory for Using Neuron Alignment} We can theoretically establish how increased similarity of intermediate activation distributions benefit mode connectivity for a specific method of alignment. Namely, we consider an alignment of the neurons minimizing the expected $L_2$ distance between their pre-activations. Then this permutation is an optimal transport plan associated with the Wasserstein-2 metric, discussed further in Appendix \ref{sec:theory_neuron_alignment}.
% It has achieved recent popularity in the machine learning community through the success of WGANs \citep{arjovsky2017wasserstein}. 
The following theorem derives tighter bounds on the loss of the linear interpolations between models after alignment compared to before alignment. Then these upper bounds extend to the optimal curves themselves.

Let the neural networks $\vtheta$ have a simple feed-forward structure as in \eqref{eq:perm_feedforward} with Lipschitz-continuous $\sigma$. The loss function $\mathcal{L}$ is also assumed to be Lipschitz-continuous or continuous with the input data distribution bounded. The function, $F$, is taken to denote the objective value of \eqref{eq:curve_find_sym} given curve parameters and permutation.  

\setcounter{topnumber}{0}

\begin{theorem}[Alignment Leads to Tighter Loss Bound]\label{thm:na_bound}
Let the above assumptions be met. Consider the following solutions to \eqref{eq:curve_find_sym}, $(\phi_u^*, \mathbf{I})$ and $(\phi_a^*, \mP)$, where the corresponding curve parameters, $\phi$, are optimal for the given permutation. Here $\mP$ is the solution to \eqref{eq:opt_trans} using the aforementioned alignment technique. Then there exists upper bounds $B_u$ and $B_a$ such that
\begin{equation}
    F(\phi_u^*, \mathbf{I}) \leq B_u, \quad F(\phi_a^*, \mP) \leq B_a \quad \text{where} \quad B_a \leq B_u.
\end{equation}
\end{theorem}
\begin{proof}
See Appendix \ref{subsec:na_theorem} for the complete proof. We provide an outline below. 

Consider the model at point $t$ along the linear interpolation between two endpoint models. It follows that the preactivations of the first layer of this model are closer (in the $W_2$ sense) to those of an endpoint after alignment than before alignment. After applying the nonlinear activation function, we have a tighter upper bound on the distance between the first layer activations of the model and those of an endpoint in the aligned case via Lipschitz continuity. Using matrix norm inequalities of the layer weights, we find a tighter bound on the distance between the preactivations of the second layer in the model and those of an endpoint model after alignment. 

We iteratively repeat this process to find a tighter upper bound on the distance between the model output to those of an endpoint model after alignment. Finally, exploiting the Lipschitz continuity of the restricted loss function, we have a tighter bound on the loss of the model output after alignment. As a tighter bound can be found at each point along the linear interpolation for aligned models, it provides a tighter bound on the associated average loss. This tighter upper bound for the linear interpolation from using alignment is then clearly a tighter upper bound for the optimal curve.         
\end{proof}

\setcounter{topnumber}{2}

The difficulty in directly comparing $F(\phi_u^*, \mathbf{I})$ and $F(\phi_a^*, \mP)$ stems from the nonlinear activation function. Thus, we bound each of these values individually. We now discuss tightness of these bounds. In the Appendix \ref{subsec:bound_tightness}, we explore what conditions need to be met for the bounds on the average loss of the linear interpolations to be tight. Given the bounds for loss along linear interpolations are tight, we can consider the tightness of the bound for loss along piecewise linear curves between networks. These piecewise linear curves can themselves approximate continuous curves. Given a piecewise linear curve with optimal loss, for two models on the same line segment of this curve, we can consider the optimal piecewise linear curve between them. It follows that this optimal curve is their linear interpolation. Therefore, we have tightness in the upper bound for curve-finding restricted to piecewise linear curves between networks of a class for which we have tightness for the linear interpolation. Via an argument by continuity, this tightness extends to continuous curves. 
% Thus, these bounds are nontrivial as we have tightness for a wide class of networks and curve parameterizations under a reasonable assumption. 

\begin{figure}[tb]
    \centering
    \begin{subfigure}[b]{0.32\linewidth}
        \centering
        \includegraphics[width=\linewidth]{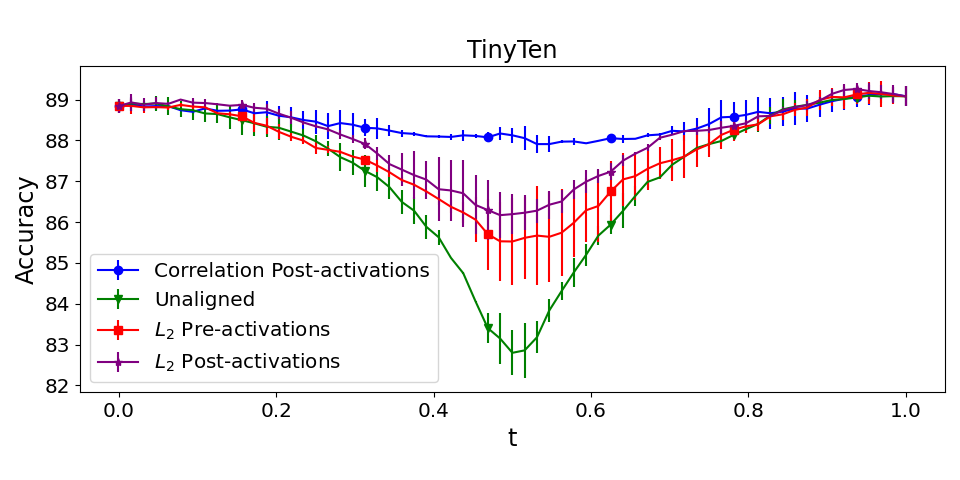}
        \caption{CIFAR10}
    \end{subfigure}
    \hfill
    \begin{subfigure}[b]{0.32\linewidth}
        \centering
        \includegraphics[width=\linewidth]{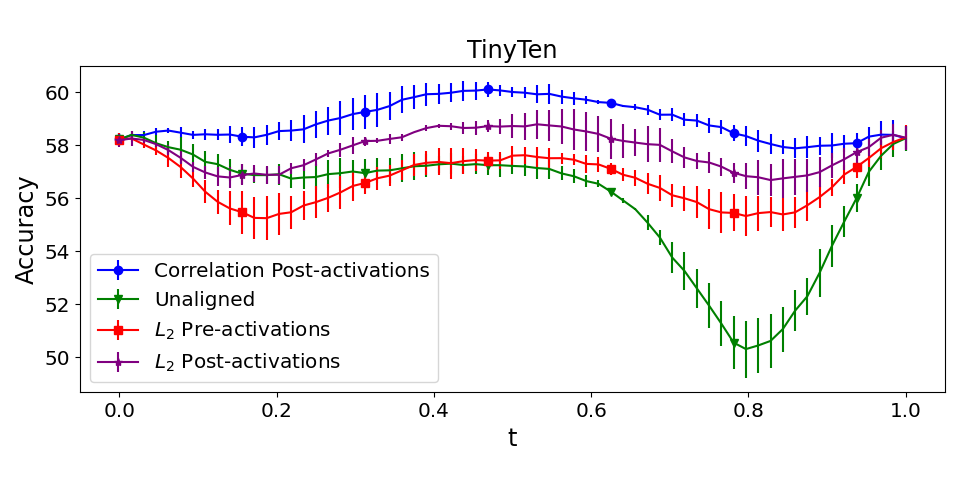}
        \caption{CIFAR100}
    \end{subfigure}
    \hfill 
    \begin{subfigure}[b]{0.32\linewidth}
        \centering
        \includegraphics[width=\linewidth]{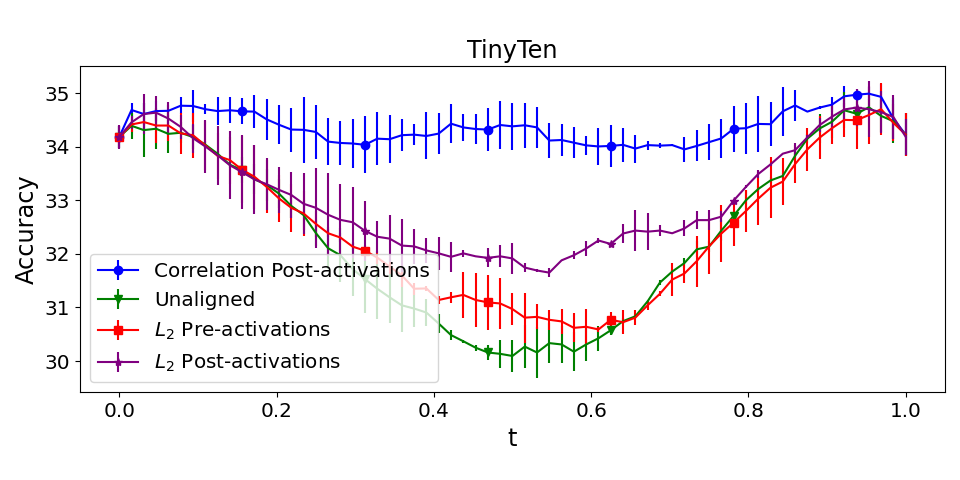}
        \caption{Tiny ImageNet}
    \end{subfigure}
    \caption{Comparison of different alignment techniques across the different datasets for the TinyTen architecture. These plots show the test accuracy along the learned curves. Notice that all methods outperform the \textit{Unaligned} case. Notably, alignment of post-activations outperforms alignment of pre-activations. Maximizing cross-correlation clearly outperforms minimizing $L_2$ distance.}
    \label{fig:align_comp}
\end{figure}

Empirically, alignment using the correlation of post-activations outperforms that of alignment associated with the Wasserstein-2 metric, as used in \citep{arjovsky2017wasserstein}, for pre-activations. See Figure \ref{fig:align_comp}. We discuss this in Appendix \ref{subsec:post_or_preact}. This observation motivates our use of the former technique throughout this paper. Additionally, cross-correlation is often preferred over unnormalized signals in certain tasks as in \citep{avants2008symmetric}. The different techniques are still fundamentally related.    

% \subsubsection{Numerical Implementation of Neuron Alignment}
\paragraph{Numerical Implementation of Neuron Alignment}

Algorithm \ref{alg:alignment} details computing a permutation of the network weights from neuron alignment. This demonstrates a correlation-based alignment of post-activations, though it easily generalizes for other techniques. In practice, we solve the assignment problem in \eqref{eq:opt_trans} using the Hungarian algorithm \citep{kuhn1955hungarian}. For an $L$ layer network with maximum width of $M$, we compute $\mP$ using a subset of the training data. Then the cost of computing the cross-correlation matrices for all layers is dominated by the forward propogation through the network to accumulate the activations. The running time needed to compute all needed assignments is $\mathcal{O}(L M^3)$, with storage $\mathcal{O}(L M)$. This is on the order of the running time associated with one iteration of forward propagation. Then neuron alignment is relatively cheap as the time complexity of computing aligned curves is on the same order as traditional curve finding. We also stress that in practice only a subset of training data is needed to compute alignment. We confirmed that the computed alignment for models trained on CIFAR100 was the same given subset sizes of (2,500, 5,000, and 10,000). This choice of a smaller subset leads to faster computation. 
% We refer to these different curves as \textit{aligned} and \textit{unaligned}.

\begin{algorithm}[tb]
\SetAlgoLined
\caption{Permutation via Neuron Alignment}
\label{alg:alignment}
\begin{algorithmic}
\STATE {\bfseries Input:} Trained Neural Networks $\vtheta_1$ and $\vtheta_2$, Subset of Training Data $\tX_0$
\STATE {\bfseries Output:} Aligned Neural Networks $\vtheta_1$ and $\bm{P} \vtheta_2$
\STATE Initialize $\mP \vtheta_2$ as $[\mW_1^2, \mW_2^2, \ldots, \mW_{L}^2]$
\FOR{each layer $l$ in $\{1, 2, \ldots, L-1\}$}
    \FOR{each network $j$ in $\{1, 2\}$}
        \STATE compute activations, $\tX_l^{(j)} = \sigma \circ \mW_{l}^j \tX_{l-1}^{(j)}$ 
        % \STATE for each element in the batch, vectorize $\tX_l^{(j)}$ if applicable 
        \STATE compute, $\mZ_l^{(j)}$, the normalization of the vectorized activations 
    \ENDFOR
    \STATE compute cross-correlation matrix, $\mC_l^{(1, 2)}$ = $\mZ_l^{(1)} \mZ_l^{(2) T}$
    \STATE compute $\mP_l$ by solving the assignment problem in eq. \ref{eq:opt_trans} with $\mC_l^{(1, 2)}$ using Hungarian algorithm 
    \STATE update $\hat{\mW}_l^2 \rightarrow \mP_l \hat{\mW}_l^2$, \quad $\hat{\mW}_{l+1}^2 \rightarrow \hat{\mW}_{l+1}^2 \mP_{l}^T$
\ENDFOR
\end{algorithmic}
\end{algorithm}

% \subsection{On Rigorously Learning the Permutation in the Joint Model}
% \label{subsec:learn_perm}
\paragraph{On Rigorously Learning the Permutation in the Joint Model}

% Traditionally, approximation algorithms are used when more rigorous techniques are intractable. 
We explore rigorously learning a more optimal permutation using a Proximal Alternating Minimization (PAM) scheme \citep{attouch2010proximal}. The PAM scheme, formalized the first equation
of Appendix \ref{sec:PAM}, solves \eqref{eq:curve_find_sym} via alternatively optimizing the weight permutation, $\mP$, and the curve parameters, $\phi$, with added proximal penalty terms. Under some assumptions, PAM guarantees global convergence to a local minima, meaning the quality of the solution is tied to the initialization. As such, PAM provides a baseline for establishing if a solution to \eqref{eq:curve_find_sym} is locally optimal. 
% \footnote{In the context of combinatorial optimization, optimal for some neighborhood of permutations. }

We implement and solve a PAM scheme starting from both the unaligned and aligned permutation initializations to gain insight into the optimality of the learned aligned curve. We establish theoretical guarantees and discuss the numerical implementation of PAM in Appendix \ref{sec:PAM}.

\paragraph{Using Alignment to Connect Adversarial Robust Models}
So far, we have discussed aligning the features of two typical neural networks to benefit learning a low loss curve connecting said networks. We are also interested in learning such curves between adversarially robust networks. This is motivated by recent observations in \citep{zhao2020bridging} that found the existence of a loss barrier between adversarial robust networks. In Section \ref{sec:robust_models}, we discuss our findings that such a barrier is in large part due to artifacts of weight permutation symmetry in the loss landscape. 

\section{Experiments}
\label{sec:experiments}

\paragraph{Datasets} We trained neural networks to classify images from CIFAR10 and CIFAR100 \citep{krizhevsky2009learning}, as well as Tiny ImageNet  \citep{imagenet_cvpr09}. The default training and test set splits are used for each dataset. Cross entropy loss of the output logits is used as the loss function. $20\%$ of the images in the training set are used for computing alignments between pairs of models. 
We augment the data using color normalization, random horizontal flips, random rotation, and random cropping. 

\paragraph{Architectures} Three different model architectures are used. They are included in Table \ref{table:results_cifar10} along with their number of parameters.
% Table \ref{table:results_cifar10} contains relevant properties of these architectures. 
The first architecture considered is the TinyTen model. TinyTen, introduced in \citep{kornblith2019similarity}, is a narrow $10$ layer convolutional neural network. This is a useful model for concept testing and allows us to gain insight to networks that are underparameterized. We also include ResNet32 \citep{he2016deep} in our experiments to understand the effect of skip connections. Finally, we consider the GoogLeNet architecture, which is significantly wider than the previously mentioned networks and achieves higher performance \citep{szegedy2015going}.
% VGG16-BN is the third architecture that we considered in our experiments \citep{simonyan2014very}. VGG16 has significantly more parameters compared to other models. 
% We chose these architectures for their varying properties and their prevalence in literature.   

All models used as curve endpoints are trained using SGD. We set a learning rate of $1\mathrm{E}{-1}$ that decays by a factor of $0.5$ every $20$ epochs. Weight decay of $5\mathrm{E}{-4}$ was used for regularization. Each model was trained for 250 epochs, and all models were seen to converge. This training scheme produced models of comparable accuracy to those in related literature such as \citep{kornblith2019similarity}. Models were trained on NVIDIA 2080 Ti GPUs.

\begin{table}[tb]
\caption[]{Average accuracy along the curve with standard deviation is reported for each combination of dataset, network architecture, and curve class. For emphasis, we list the performance of the worst model found along the curve. The average accuracy of the endpoint models are also included. Aligned curves clearly outperform the unaligned curves of \citep{garipov2018loss}. Number of parameters is also included. 
% Note that aligned accuracies are typically as high as the trained model accuracies used as endpoints. 
}
% \vspace{-2mm}
\label{table:results_cifar10}
\small
\adjustbox{max width=\textwidth}{
% \begin{center}
\centering
\begin{tabular}{@{} r r r c r r c r r @{}} 
% \hline
\toprule
\multirow{2}{4em}{Model} & \multicolumn{8}{c}{Average/Minimum Accuracy of Models Along the Learned Curve} \\
& \multicolumn{2}{c}{CIFAR10} & \phantom{a} & \multicolumn{2}{c}{CIFAR100} & \phantom{a} & \multicolumn{2}{c}{Tiny ImageNet}\\
% \cmidrule{2-3} \cmidrule{5-6} \cmidrule{8-9} 
% \cline{4-7}
% & Avg. & Min. && Avg. & Min. && Avg. & Min. \\  
\midrule
TinyTen (0.09M)  & $89.0 \pm 0.1$ &&& $58.1 \pm 0.5$ &&& $34.2 \pm 0.2$ & \\
\cmidrule{2-3} \cmidrule{5-6} \cmidrule{8-9}
Unaligned (Garipov) & $87.4 \pm 0.1$ & $82.8 \pm 0.5$ && $56.0 \pm 0.2$ & $53.2 \pm 1.1$ && $32.5 \pm 0.1$ & $30.0 \pm 0.3$\\
PAM Unaligned & $87.6 \pm 0.1$  & $84.0 \pm 0.3$ && $57.3 \pm 0.2$ & $55.9 \pm 0.9$ && $33.6 \pm 0.1$ & $32.5 \pm 0.1$ \\
PAM Aligned & $88.4 \pm 0.1$ & $87.6 \pm 0.2$ && $\mathbf{58.8 \pm 0.1}$ & $\mathbf{57.7 \pm 0.3}$ && $34.2 \pm 0.1$ & $33.4 \pm 0.2$ \\
Aligned & $\mathbf{88.5 \pm 0.1}$ & $\mathbf{87.8 \pm 0.1}$ && $58.7 \pm 0.2$ & $\mathbf{57.7 \pm 0.4}$ && $\mathbf{34.4 \pm 0.1}$ & $\mathbf{33.7 \pm 0.1}$ \\
\midrule
ResNet32 (0.47M) & $92.9 \pm 0.1$ &&& $67.1 \pm 0.5$ &&& $50.2 \pm 0.0$ & \\
\cmidrule{2-3} \cmidrule{5-6} \cmidrule{8-9}
Unaligned (Garipov) & $92.2 \pm 0.1$ & $89.1 \pm 0.2$ && $66.5 \pm 0.2$ & $64.7 \pm 0.4$ && $48.2 \pm 0.1$ & $45.2 \pm 0.1$ \\ 
PAM Unaligned & $92.4 \pm 0.1$ & $89.9 \pm 0.2$ && $67.0 \pm 0.1$ & $66.1 \pm 0.1$ && $48.5 \pm 0.1$ & $46.6 \pm 0.1$ \\
PAM Aligned & $\mathbf{92.7 \pm 0.0}$ & $92.1 \pm 0.1$ && $67.6 \pm 0.4$ & $\mathbf{66.8 \pm 0.1}$ && $49.2 \pm 0.4$ & $47.9 \pm 0.6$\\
Aligned & $\mathbf{92.7 \pm 0.1}$ & $\mathbf{92.2 \pm 0.0}$ && $\mathbf{67.7 \pm 0.1}$ & $66.6 \pm 0.1$ && $\mathbf{49.5 \pm 0.3}$ & $\mathbf{48.8 \pm 0.4}$ \\
\midrule
GoogLeNet (10.24M) & $93.4 \pm 0.0$ &&& $73.2 \pm 0.4$ &&& $51.6 \pm 0.2$ & \\
\cmidrule{2-3} \cmidrule{5-6} \cmidrule{8-9}
Unaligned (Garipov) & $93.3 \pm 0.0$ & $92.1 \pm 0.1$ && $73.1 \pm 0.4$ & $69.8 \pm 0.3$ && $51.9 \pm 0.1$ & $48.7 \pm 0.4$ \\ 
Aligned & $\mathbf{93.4 \pm 0.0}$ & $\mathbf{93.1 \pm 0.0}$ && $\mathbf{73.4 \pm 0.3}$ & $\mathbf{72.9 \pm 0.3}$ && $\mathbf{52.4 \pm 0.2}$ & $\mathbf{51.4 \pm 0.3}$ \\
% \hline
\bottomrule
\end{tabular}
% \end{center}
}
\vspace{-2mm}
\end{table}
% \footnotetext{Strictly speaking, the algorithm converges to a local optima in the convex relaxation of the domain. The learned permutation is the projection of this optima to the feasible set.}

\paragraph{Training curves}
All curves are parameterized as quadratic Bezier curves. Bezier curves are defined by their \textit{control points}. In the current study, we refer to  endpoint models as $\vtheta_1$ and $\vtheta_2$ as well as the control point, $\vtheta_c$. Then $\vr$ is defined in \eqref{eq:Bezier} with $\vtheta_c$ as the learnable parameter in $\phi$
\begin{equation}
    \vr_\phi(t) = (1 - t)^2 \vtheta_1 + 2 (1 - t) t \vtheta_{c} + t^2 \vtheta_2. \label{eq:Bezier}
\end{equation}
% Important properties of the quadratic Bezier curve include $\vr(0) = \vtheta_1$, $\vr(1) = \vtheta_2$, $\vr'(0) = 2(\vtheta_c - \vtheta_1)$, and $\vr'(1) = 2(\vtheta_2 - \vtheta_c)$. 
% Then $\vtheta_c$ is the learnable parameter in $\phi$. 
% Of course one could consider more complicated curve parameterizations. In practice, we find a simple curve to be enough for our experiments, and consider the learning of a planar curve along which loss is nearly constant to be significant in itself. 

% \subsection{Training Curves}
% \label{subsec:train_curves}

For each architecture and dataset, we train $6$ models using different random initializations. Thus we have $3$ independent model pairs. 
We learn two classes of curves, \textit{Unaligned} / \textit{Aligned},  that are solutions to algorithm \ref{alg:curve_find} for $\vtheta_1$ and $\vtheta_2$ / $\mP \vtheta_2$. Here $\mP$ is the permutation learned by alignment. We also learn the corresponding curve classes, \textit{PAM Unaligned} / \textit{PAM Aligned}, that are solutions to the first equation in Appendix \ref{sec:PAM}. Their permutation intializations are $\mI$ and $\mP$ respectively. Tables are generated with curves trained from different random seeds, while figures are generated with curves trained from the same seed. Curves trained from different seeds are similar up to symmetry, but using them to generate the figures would average out shared features.

% We learn four classes of curves:
% \begin{itemize}
%     \setlength\itemsep{0em}
%     \item Unaligned: Solution to algorithm \ref{alg:curve_find} for $\vtheta_1$ and $\vtheta_2$
%     \item PAM: Solution to \eqref{eq:PAM} for $\vtheta_1$ and $\vtheta_2$ with $\mP^{(0)} = \mI$
%     \item PAM Aligned: Solution to \eqref{eq:PAM} for $\vtheta_1$ and $\vtheta_2$ with $\mP^{(0)} = \mP$
%     \item Aligned: Solution to algorithm \ref{alg:curve_find} for $\vtheta_1$ and $\mP \vtheta_2$
% \end{itemize}
% where $\mP$ denotes the permutation learned by neuron alignment (algorithm \ref{alg:alignment}).

Curves are trained for 250 epochs using SGD with a learning rate of $1\mathrm{E}{-2}$ and a batch size of 128. The rate anneals by a factor of 0.5 every 20 epochs. The initial hyperparameters were chosen to match those in \citep{garipov2018loss}. 
For CIFAR100 curves, this learning rate was set to $1\mathrm{E}{-1}$, as they were seen to perform marginally better. 
For the TinyTen/CIFAR100 case, we explore the hyperparameter space in Appendix \ref{sec:hyperparam_search} to show that our results hold under different choice of hyperparameters.  

% \begin{itemize}
%     \setlength\itemsep{0em}
%     \item Unaligned: Solution to algorithm \ref{alg:curve_find} for $\vtheta_1$ and $\vtheta_2$
%     \item PAM Unaligned: Solution to \eqref{eq:PAM} for $\vtheta_1$ and $\vtheta_2$ with $\mP^{(0)} = \mI$
%     \item PAM Aligned: Solution to \eqref{eq:PAM} for $\vtheta_1$ and $\vtheta_2$ with $\mP^{(0)} = \mP_{Al}$
%     \item Aligned: Solution to algorithm \ref{alg:curve_find} for $\vtheta_1$ and $\mP_{Al} \vtheta_2$
% \end{itemize}
% where $\mP_{Al}$ denotes the permutation learned by neuron alignment (algorithm \ref{alg:alignment}).

% We learn PAM curves for all architectures except VGG16, as its size made this computationally prohibitive. We train two sets of each curve class. One set involves the curves learned when the random seed for curve finding is fixed for all model pairs. The other set consists of the curves learned when the random seed is different for each model pair. We find that the learned curves for different seeds are similar up to reindexing the endpoints. For Figures \ref{fig:viz_curve}, \ref{subfig:acc_plane_cifar100}, and \ref{subfig:acc_curve_plane_cifar100}, we use the first set of curves so that interesting geometric features on the loss surface are not averaged out. For tables and other figures, we use the second more general set of curves.  

% \subsection{Aligned Curves Outperform Unaligned Curves}
\subsection{Results on using Neuron Alignment }
\label{subsubsec:align_trad_curve}

\paragraph{Aligned Curves Outperform Unaligned Curves} 
% We investigate using neuron alignment as a heuristic for curve finding up to symmetry. 
The test accuracy of learned curves can be seen for each dataset, architecture, and curve class in Table \ref{table:results_cifar10}. Clearly, the aligned curves outperform the unaligned curve. In many cases, the average accuracy along the aligned curves in comparable to the trained models used as endpoints. The table also contains the minimum accuracy along the curve, indicating  that aligned curves do not suffer from the same generalization gap that unaligned curves are prone to. Finally, Table \ref{table:results_cifar100_loss_train} in the appendix contains the training loss for each case at convergence. Overall, the strongest gains from using alignment are for underparameterized networks. As seen in Table \ref{table:results_cifar10}, the largest increase in performance is for TinyTen on Tiny ImageNet while the smallest gain is made for GoogLeNet on CIFAR10. This is inline with observations by \citep{freeman2016topology}.

The test loss and accuracy along the learned curves for CIFAR100 are shown in Figure \ref{fig:viz_curve}. We observe that, as expected, the accuracy at each point along the aligned curve usually exceeds that of the unaligned curve, while the loss along the curve is also smoother with neuron alignment. 
% \footnote{See Figure \ref{fig:fft} in the appendix for a more quantitative measure of smoothness via the Fourier transform.} 
Of note is the prominent presence of an accuracy barrier along the unaligned curves. 
% Overall, loss along the aligned curves varies more smoothly and has better generalization.

% Noteworthy is the prominent presence of the accuracy barrier along the unaligned curve around $t$ at $0.8$ for all models. This accuracy barrier corresponds to a clear loss barrier for Tiny-10 and ResNet32. In contrast, for VGG16 there is lowest loss at this point on the unaligned curve with worse generalization performance. Overall, loss along the aligned curves varies more smoothly and has better generalization.

% We do expect the gain in performance from alignment to decrease as the parameterized curves are allowed to become more complex. However, simpler parameterized curves are appealing from a modeling perspective and they are less expensive to train. 
% \PD{the last sentence seems out of context to me}

Figure \ref{fig:viz_curve} displays the planes containing the initializations for curve finding. Clearly the aligned initialization has better objective value. The axis is determined by Gram-Schmidt orthonormalization. In Appendix \ref{sec:add_figs}  the planes containing the learned curves are displayed in Figures \ref{fig:loss_acc_curve_planes_add}. 
% These are the planes containing $\vtheta_1$, $\mP \vtheta_2$, and $\vtheta_c$, although the control point is out of bounds of the figure.  
The loss displayed on the planes containing the linear initializations and the curves can be seen in Figure \ref{fig:plane_loss_additional}.

Practically, using neuron alignment for determining the permutation $\mP$ may be enough and  avoids  more complicated optimization.
%(although one may still be interested in  theoretically. 
Note the relative flatness of the accuracy along the aligned curves in Figure \ref{fig:viz_curve}. Additionally, the plots in the top row indicate much faster convergence when learning $\phi$  using neuron alignment. For example, the aligned curve takes 100 epochs less to achieve the training accuracy that the unaligned curve converges to, for TinyTen and CIFAR100. 
% Even for VGG16, the aligned curve reaches the milestone $40$ epochs earlier. 
% Additionally, there is clearly a gap in the accuracy that the curves converge to, with the aligned curve always outperforming the unaligned one, while the underparameterized architectures receive a more significant accuracy boost. F
Figures \ref{fig:training_cifar10} and \ref{fig:loss_acc_curve_planes_add} in Appendix \ref{sec:add_figs} displays the previously mentioned plots for the additional datasets. 

% While these observations are promising, we intend to provide insight into why neuron alignment works. 
% \subsubsection{Similarity of Intermediate Activations Along the Curve}
\paragraph{Similarity of Intermediate Activations Along the Curve}
In Figures \ref{subfig:corr_end_unalignmid} and \ref{subfig:corr_end_alignmid}, the plots display the similarity of intermediate activations between the learned curve midpoints and the endpoint models. This is evidence that alignment increases this similarity. Notice that even when the endpoints are not aligned, the learned midpoint on the unaligned curve is seen to be mostly aligned to the endpoints. Thus, the similarity of these distributions is important to mode connectivity, and we can take the view that these learned curves are trying to smoothly interpolate features. 
% Additionally, the midpoint on the aligned curve is seen to already be optimally aligned to the endpoints. 

\begin{figure}[tb]
    \centering
    \begin{subfigure}[b]{0.53\linewidth}
        \centering
        \includegraphics[width=1.0\linewidth]{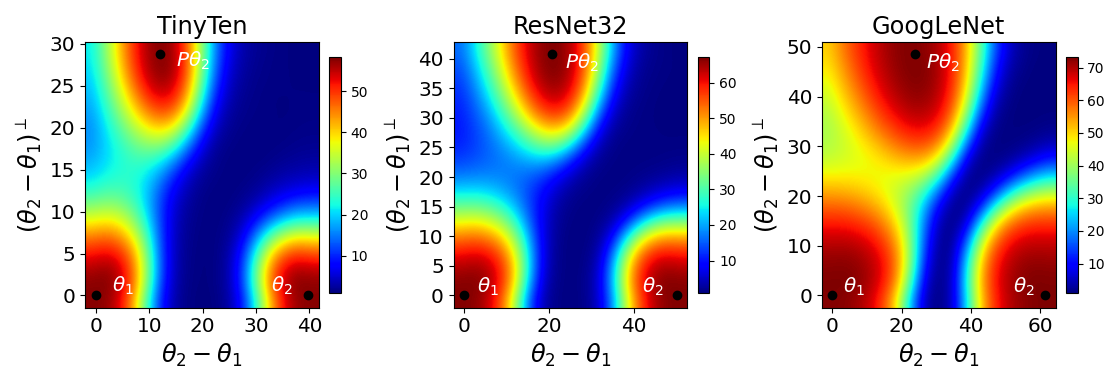}
        \hfill
        \includegraphics[width=0.9\linewidth]{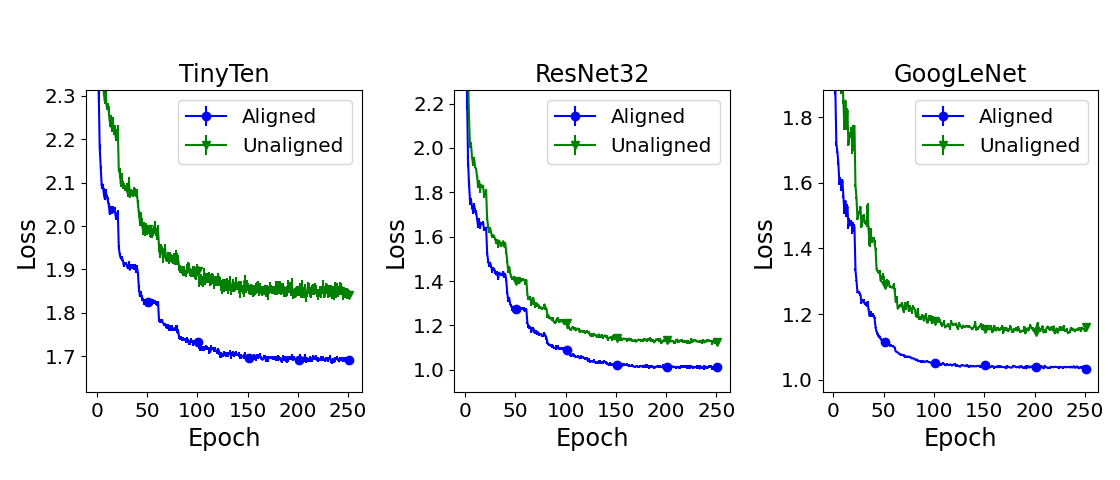}
    \end{subfigure}
    \hfill 
    \begin{subfigure}[b]{0.45\linewidth}
    \begin{subfigure}[b]{1.0\linewidth}
        \centering
        \includegraphics[width=\linewidth]{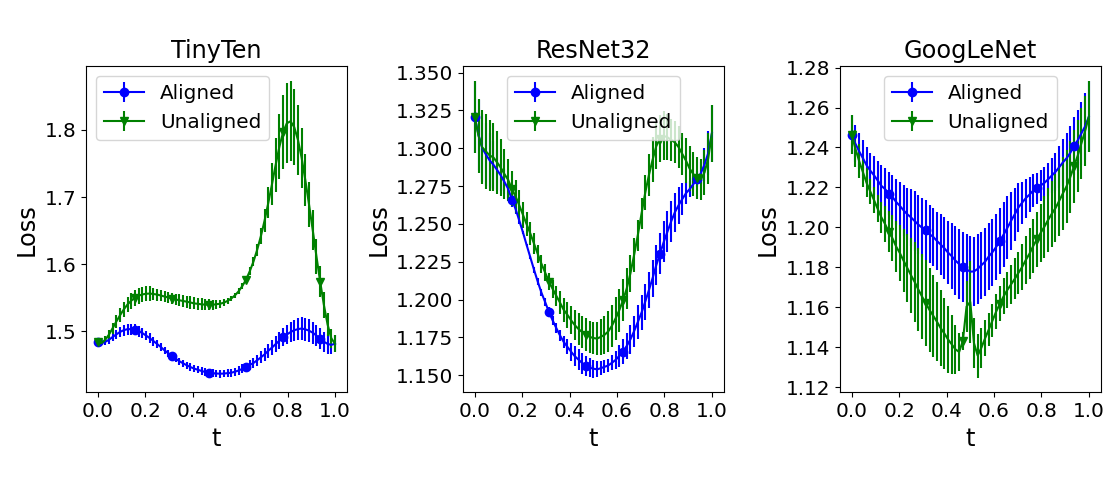}
        %\caption{CIFAR100 loss along curve}
        % \caption{Cross Entropy Loss Along Curve}
        % \label{subfig:loss_on_curve}
    \end{subfigure}
    \vfill
    % \vspace{-2mm}
    \begin{subfigure}[b]{1.0\linewidth}
        \centering
        \includegraphics[width=\linewidth]{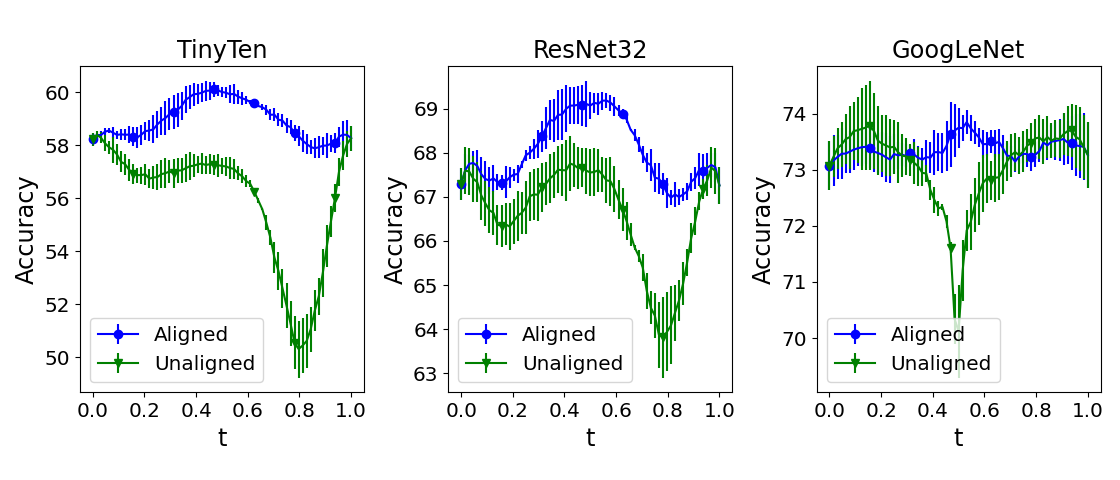}
        %\caption{CIFAR100 accuracy along curve}
        % \caption{Accuracy Along Curve}
        % \label{subfig:acc_on_curve}
    \end{subfigure}
    \end{subfigure}
    % \vspace{-2mm}
    \caption{
    % \textbf{Top}: The training loss for learning the quadratic Bezier curve between model endpoints on CIFAR100. These are compared for aligned and unaligned curves. The training of aligned curves converges to lower loss value in less epochs than for unaligned curves. 
    \textbf{Top Left:} Test accuracy on CIFAR100 across the plane containing $\vtheta_1$, $\vtheta_2$, and $\mP \vtheta_2$. This plane contains the two different intializations used in our curve finding experiments. The default initialization, $\vtheta_2 - \vtheta_1$, and the aligned initialization, $\mP \vtheta_2 - \vtheta_1$. This shows that the aligned initialization is notably better. \textbf{Bottom Left}: Training loss during training. \textbf{Top/Bottom Right:} Test loss/accuracy along these curves. Aligned curves generalize better and do not suffer from large drops in accuracy typical for unaligned curves.}
    \label{fig:viz_curve}
    \vspace{-2mm}
\end{figure}

\paragraph{Insights on Local Optimality via PAM}
As seen in Table \ref{table:results_cifar10}, starting from the unaligned initialization, it is possible to find a better permutation for optimizing mode connectivity. In contrast, the permutation given by alignment is nearly locally optimal as \textit{PAM Aligned} and \textit{Aligned} are comparable. Even more, the solutions to PAM with an unaligned initialization, \textit{PAM}, still do not perform better than curve finding with the aligned permutation, \textit{Aligned}. This implies that the neuron alignment heuristic is not only inexpensive, it provides a permutation of optimality not likely found without conducting some kind of intractable search. We see this as a strength of neuron alignment. 

\section{Mode Connectivity of Adversarial Robust Models}
% \section{New Findings for Robust Mode Connectivity}
\label{sec:robust_models}

\begin{figure}[tb]
    \centering
    % \begin{subfigure}[b]{0.45\linewidth}
    %     \centering
    %     \includegraphics[width=\linewidth]{figures/cifar100_robusttrainingloss.png}
    %     %\caption{CIFAR100 loss along curve}
    %     % \caption{Clean Accuracy Along Curve}
    %     % \label{subfig:cifar100_clean_acc}
    % \end{subfigure}
    %     % \vspace{-2mm}
    % \hfill
    \begin{subfigure}[b]{0.48\linewidth}
    % \begin{subfigure}[b]{1.0\linewidth}
        \centering
        \includegraphics[width=\linewidth]{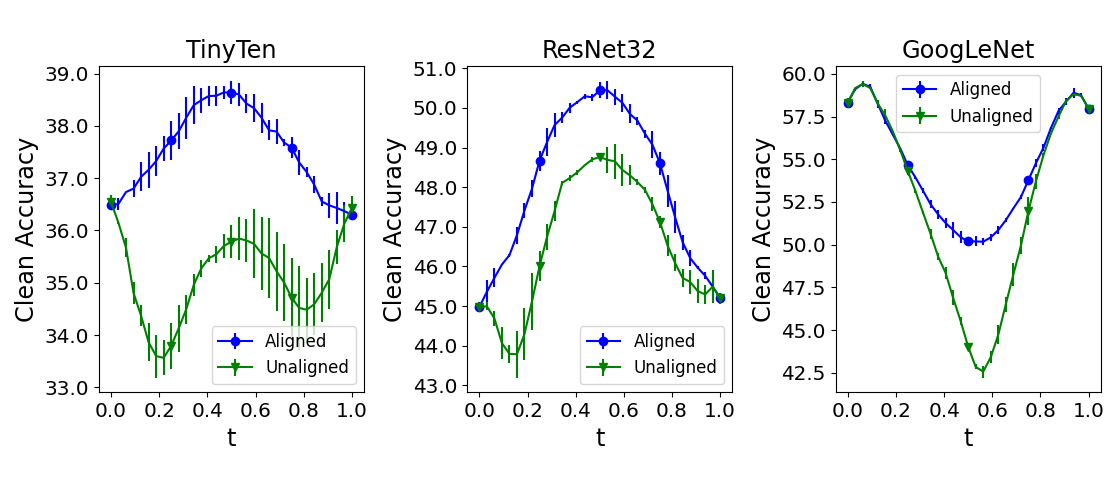}
        %\caption{CIFAR100 loss along curve}
        % \caption{Clean Accuracy Along Curve}
        % \label{subfig:cifar100_clean_acc}
    \end{subfigure}
    \hfill
    % \vspace{-2mm}
    \begin{subfigure}[b]{0.48\linewidth}
        \centering
        \includegraphics[width=\linewidth]{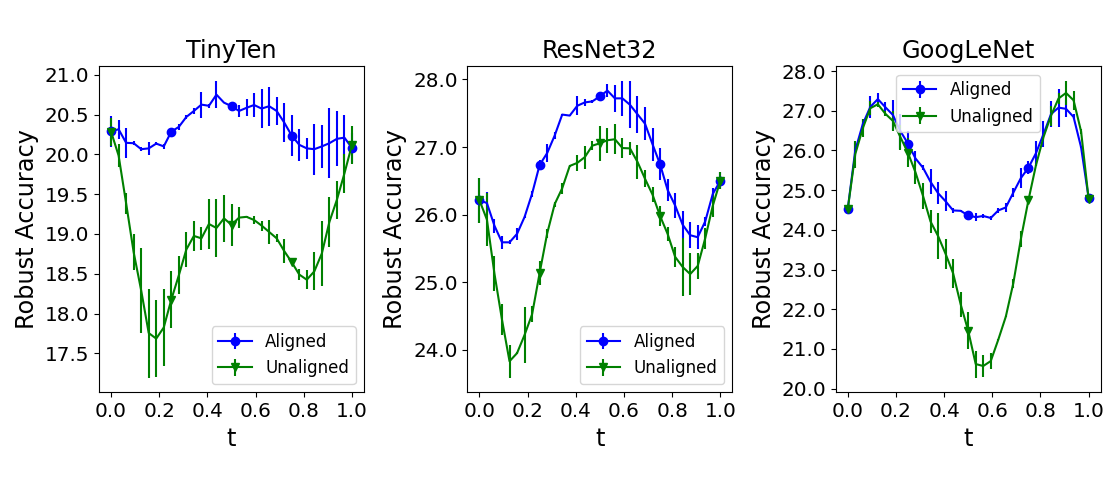}
        % \caption{Robust Accuracy Along Curve}
        % \label{subfig:cifar100_robust_acc}
    \end{subfigure}
    % \end{subfigure}
        %\vspace{-4mm}
    % \caption{\textbf{Top}: The robust training loss for learning the robust quadratic Bezier curve between robust model endpoints on CIFAR100. By robust loss, this means that the input undergoes a PGD attack during evaluation. This shows that alignment decreases this training loss. \textbf{Middle/Bottom}: Clean/Robust test accuracy along these curves. For TinyTen and ResNet32, it is clear that a more robust and accurate model can be found along the curve compared to the endpoints. VGG16 does not exhibit this behavior due to overfitting to attacks on the training data.}
    \caption{\textbf{Left/Right}: Clean/Robust test accuracy along these curves. For TinyTen and ResNet32, it is clear that a more robust and accurate model can be found along the curve compared to the endpoints. GoogLeNet results are more complicated due to well known overfitting of overparameterized models during adversarial training \citep{rice2020overfitting}. }
    \label{fig:cifar100_robust}
    \vspace{-2mm}
\end{figure}

\paragraph{Alignment greatly reduces the robust loss barrier} 
% Here we refer to the cross-entropy loss on the original samples and samples perturbed via PGD attacks as clean and robust loss respectively.
Figure \ref{fig:cifar100_robust} displays the training loss and test accuracy of the learned robust curve between adversarially trained robust CIFAR100 models for the previously mentioned architectures. The left plots displays standard test accuracy whereas the right plots displays the test accuracy of PGD adversarial examples. These networks and curves are trained with the same scheme as in \citep{zhao2020bridging}, where we set the initial learning rate to $1\mathrm{E}{-1}$. An important point to emphasize is that the curve itself is trained to minimize robust loss, so the input undergoes PGD attack at each point along the curve. 

For the robust curve learned between two unaligned robust models, we encounter barriers both in clean and robust accuracy, as reported in \citep{zhao2020bridging}. As in Figure \ref{fig:viz_curve}, these accuracy barriers appear to correspond with barriers in loss, where plots of robust loss along these curves can be found in Figure \ref{fig:robust_loss} in Appendix \ref{sec:add_figs}. It is clear that the barrier in clean accuracy is eliminated or greatly reduced with the use of alignment. With respect to robust accuracy, we see that alignment significantly alleviates that barrier. In practice, adversarial training of the GoogLeNet curves were found to overfit on the training data. This has been well observed for over-parameterized models \citep{rice2020overfitting}. Thus, the curves displayed for GoogLeNet are the results with early stopping after roughly 20 epochs to prevent overfitting. Figure \ref{fig:robust_cifar10} in the Appendix \ref{sec:add_figs} displays these results for CIFAR10.
% , and Figure \ref{fig:robust_loss} displays the robust loss evaluated along the curves. Additionally, Figure \ref{fig:robust_align} shows the correlation signatures on the robust TinyTen models. 
% \PY{We talked a lot about reducing "loss" barrier but infact what we show is accuracy, which may give reviewers hard time to digest. Can we also put robust losses in Appendix and mention we show accuracy in main paper (as they are directly related to robustness evaluation metrics)?}
%\subsection{Finding More Accurate Robust Models}

\paragraph{Aligned curves can find more accurate robust models.} 
Neuron alignment seems successful at finding a curve between robust models along which models maintain their robustness to PGD attacks without sacrificing clean accuracy. Results provide  evidence that the presence of a large robust loss barrier between robust models as mentioned in \citep{zhao2020bridging} can mostly be attributed as an artifact of symmetry in the loss landscape resulting from the network weights. 
% A caveat being that the training of individual robust models must be able to successfully generalize on test data undergoing PGD attack. 

For CIFAR100, alignment enables finding a more accurate model without sacrificing robust accuracy, which provides new insight towards overcoming the issue of robustness-accuracy tradeoff in adversarial robustness \citep{su2018robustness}. Consider the midpoint on the ResNet32 aligned curve in Figure \ref{fig:cifar100_robust}, where both clean and robust accuracies increase by $5.3\%$ and $1.3\%$, respectively, in comparison to the endpoints. For TinyTen, these accuracies also increase at the aligned curve midpoint, while no better model in term of clean or robust accuracy exists along the unaligned curve with respect to the endpoints. For GoogLeNet, we find comparable more accurate and robust models near the endpoints for both curves, though only the aligned curve avoids a robust accuracy barrier. We emphasize that learning a better model from scratch  is not an easy task. In practice, converging here requires a step size large enough for the SGD trajectory to reach this basin within feasible time and that is then small enough to converge in the basin. Thus, the aligned curve finding can be viewed as a technique for avoiding hyperparameter tuning, which is typically expensive.

\section{Conclusion}
\label{sec:discussion}
We generalize mode connectivity by removing the weight symmetry ambiguity associated with the endpoint models. The optimal permutation of these weights can be approximated using neuron alignment. We empirically find that this approximation is locally optimal and outperforms the locally optimal solution to a proximal alternating scheme with random initialization. Empirically, we show that neuron alignment can be used to successfully and efficiently learn optimal connections between neural nets. Addressing symmetry is critical for learning planar curves on the loss surface along which accuracy is mostly constant. Our results hold true over a range of datasets and network architectures. With neuron alignment, these curves can be trained in less epochs and to higher accuracy. Novel to previous findings, with alignment we also find that robust models are in fact connected on the loss surface and curve finding serves as a means to identify more accurate robust models. 
% Future work  will include gaining a deeper \textit{theoretical understanding} of how neuron alignment affects curve finding dynamics. We plan to further explore how similar feature representations between models are smoothly interpolated by these learned curves, and how this relates to network training. 

% \clearpage

\section*{Broader Impact}
This work examines solving the problem of mode connectivity up to a symmetry in the weights of the given models. Our method allows for the computation of a curve of nearly optimal models, where this curve itself has a simple parameterization. We discuss the broader impacts of this work from the following perspectives:
\begin{itemize}[leftmargin=1cm]
    \item \textbf{Who may benefit from this research: } In this work we show the ability to learn simply parameterized curves along which each model is seen to be robust. This could have potential applications in ensembling, where an ensemble of robust networks can be learned without training each individual model. Generally, an attack on an ensemble will be less effective than on a direct attack on any of its component models. Additionally, regarding CIFAR100, we see the ability to find comparable robust models of greater accuracy along the curve. Thus, this work can benefit systems for which robustness is critical. 
    \item \textbf{Who may be disadvantaged from this research:} As mentioned, this work can benefit systems for which robustness is critical. Such systems are typically part of a movement towards \textit{trusted artificial intelligence}. As trust in systems increases, these systems may see wider use and adoption, such as self-driving cars. With increased automation, workers such as truck drivers stand to have declining career prospects. Thus, this work is part of a broader push in research that may disadvantage these peoples.   
    \item \textbf{Consequences of failure: } If our method fails for a given instance, then it means that we were unable to learn a simple curve along which the models are nearly optimal. This means the method cannot be used in that instance for an application such as providing a set of models for ensembling. In the case of adversarial models, this means that the learned models are vulnerable to attacks and becoming compromised.  
    \item \textbf{Biases in the data:} In our experiments, we validated our results for three different datasets to confirm that our method does not depend on a bias uniquely associated with an individual dataset. 
\end{itemize}

\begin{ack}
This work was supported by the Rensselaer-IBM AI Research Collaboration (http://airc.rpi.edu), part of the IBM AI Horizons Network (http://ibm.biz/AIHorizons). 
Additionally, R. Lai’s work is supported in part by NSF CAREER Award (DMS—1752934). The authors also thank Youssef Mroueh for helpful discussions on optimal transport.  
\end{ack}

% \clearpage
%\appendix
\bibliography{main}
\bibliographystyle{icml2020}

\clearpage

\appendix

\section{Additional Figures}
\label{sec:add_figs}

\begin{figure}[htb]

    \begin{subfigure}[b]{0.48\linewidth}
        \centering
    \begin{subfigure}[b]{1.0\linewidth}
        \centering
        \includegraphics[width=\linewidth]{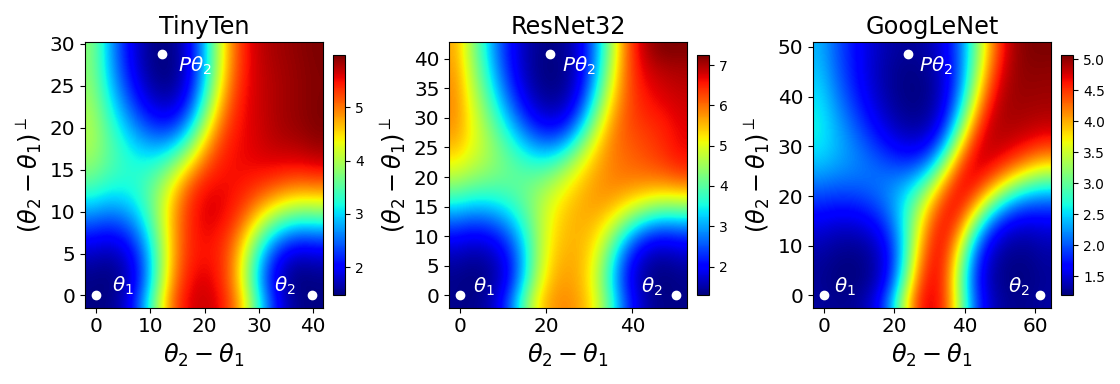}
        \caption{CIFAR100}
    \end{subfigure}
    \begin{subfigure}[b]{1.0\linewidth}
        \centering
        \includegraphics[width=\linewidth]{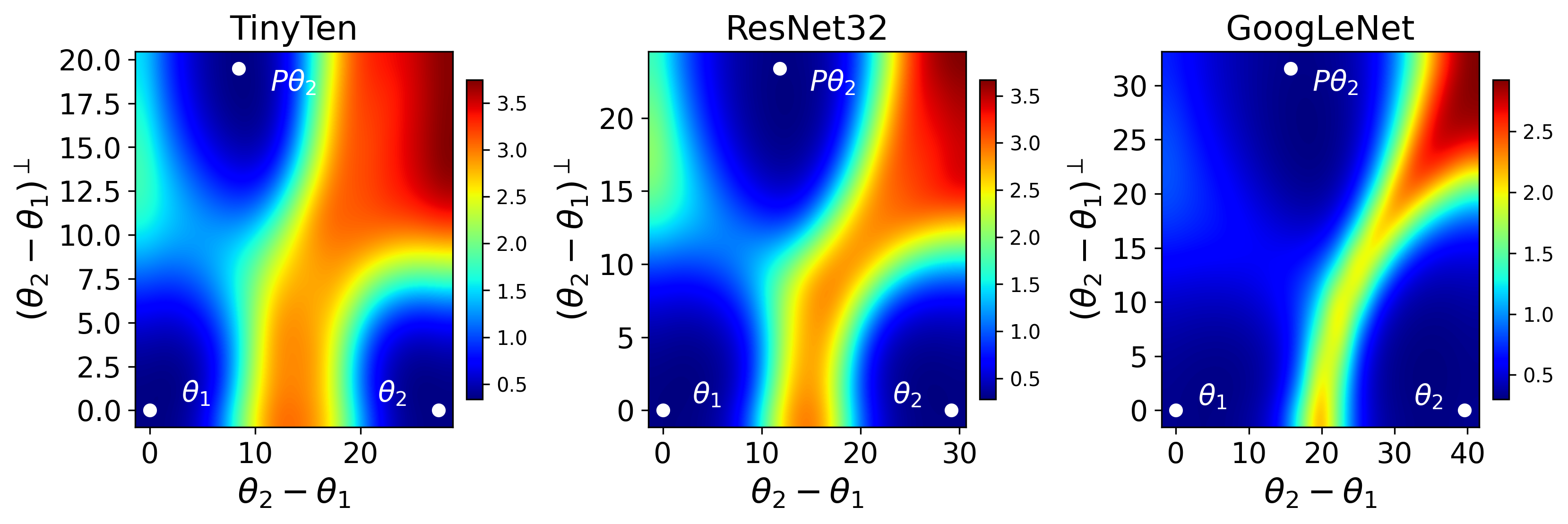}
        \caption{CIFAR10}
    \end{subfigure}
    \begin{subfigure}[b]{1.0\linewidth}
        \centering
        \includegraphics[width=\linewidth]{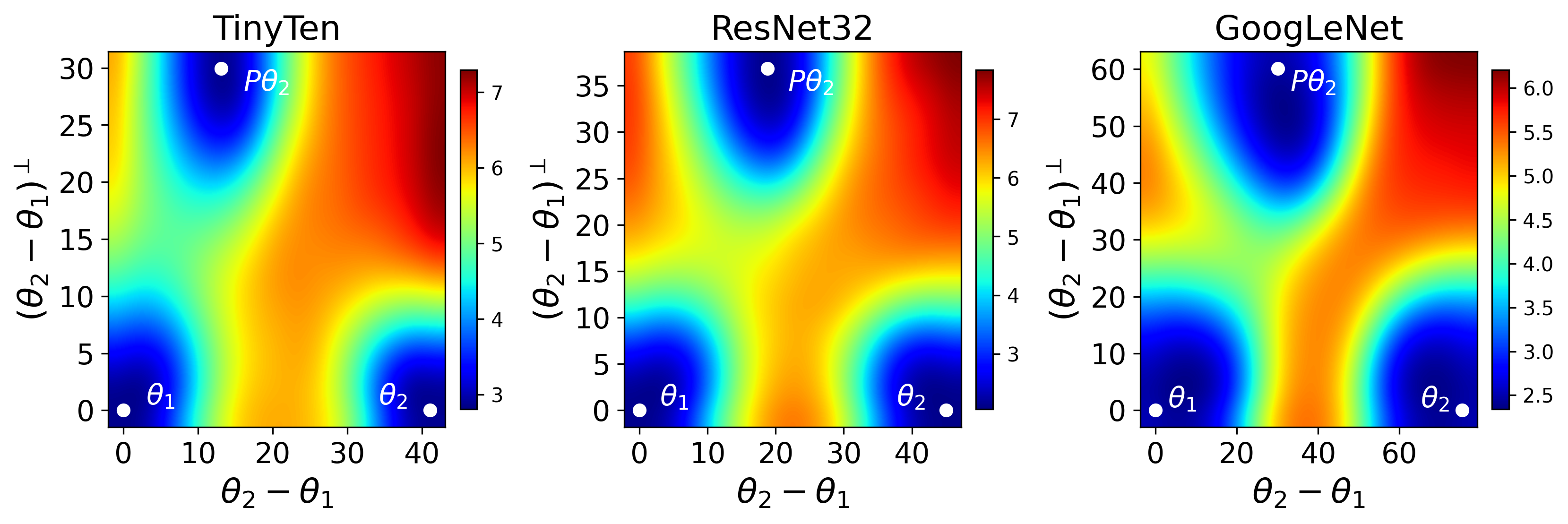}
        \caption{Tiny ImageNet}
    \end{subfigure}
    \end{subfigure}
    \hfill
    \begin{subfigure}[b]{0.48\linewidth}
        \begin{subfigure}[b]{1.0\linewidth}
        \centering
        \includegraphics[width=\linewidth]{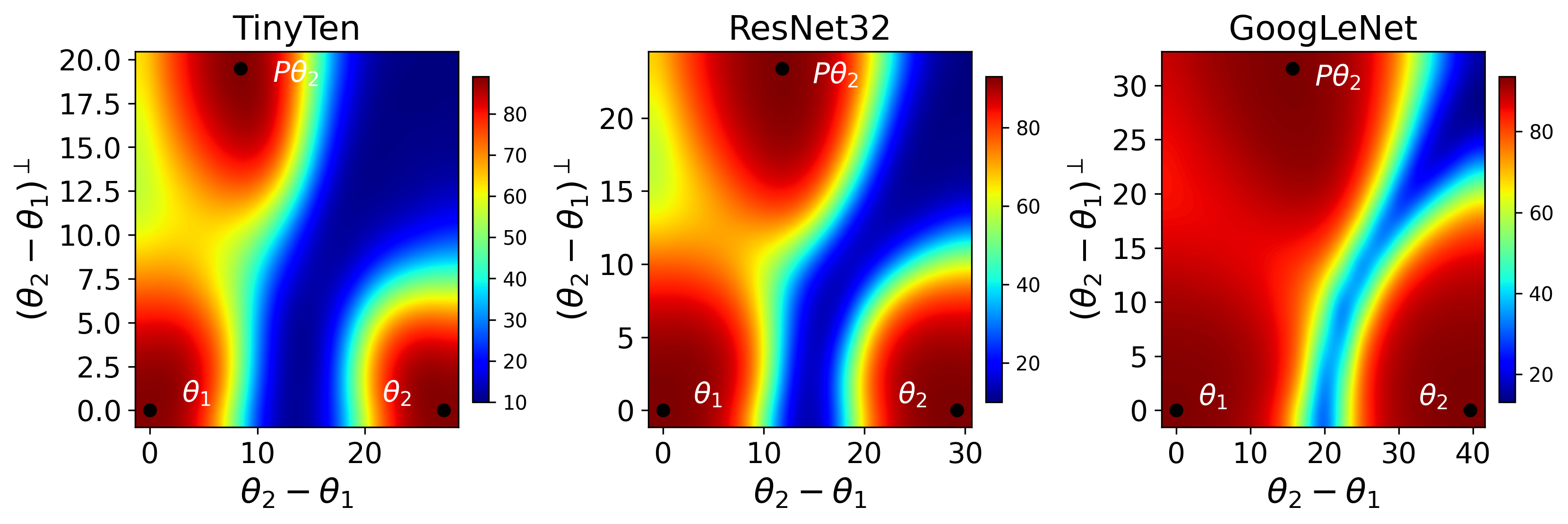}
        \caption{CIFAR10}
    \end{subfigure}
    \begin{subfigure}[b]{1.0\linewidth}
        \centering
        \includegraphics[width=\linewidth]{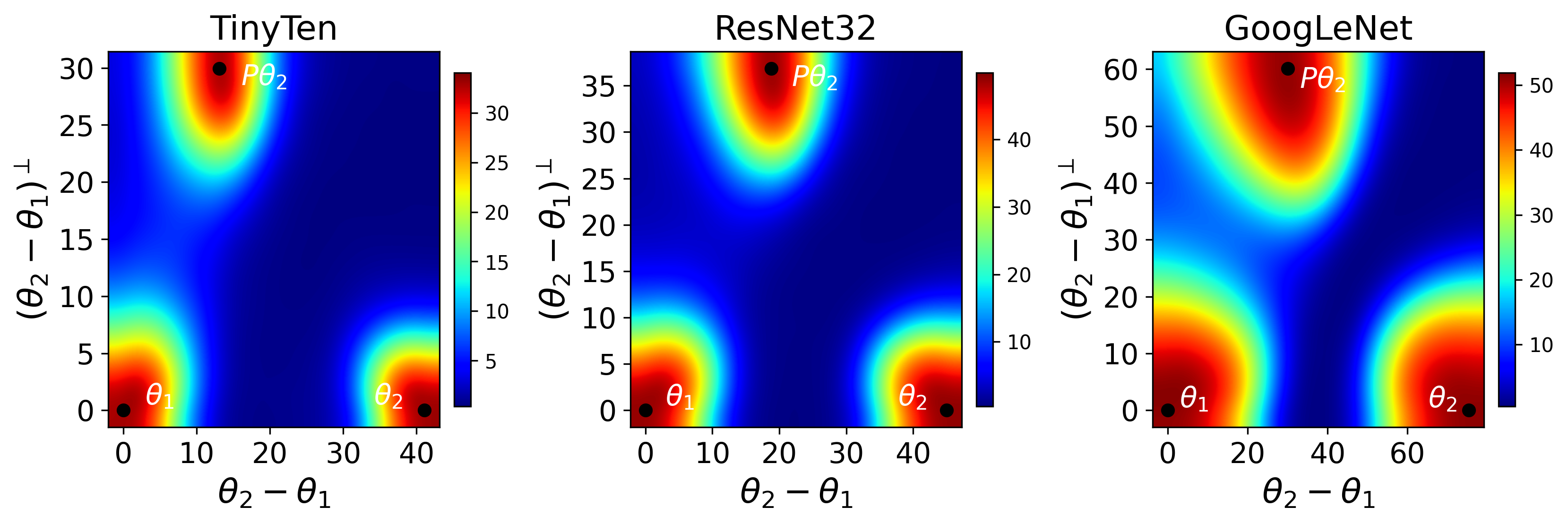}
        \caption{Tiny ImageNet}
    \end{subfigure}
    \end{subfigure}
    \vspace{-2mm}
    \caption{\textbf{Left/Right:} Test loss/accuracy on plane containing $\vtheta_1$, $\vtheta_2$, and $\mP \vtheta_2$.}
    \label{fig:plane_loss_additional}
\end{figure}

% \begin{figure}[H]
%     \centering
%     \begin{subfigure}[b]{0.95\linewidth}
%         \centering
%         \includegraphics[width=\linewidth]{figures/fin_plane_acc_CIFAR10.png}
%         \caption{CIFAR10}
%     \end{subfigure}
%     \begin{subfigure}[b]{0.95\linewidth}
%         \centering
%         \includegraphics[width=\linewidth]{figures/fin_plane_acc_STL10.png}
%         \caption{STL10}
%     \end{subfigure}
%     \vspace{-2mm}
%     \caption{Test accuracy on plane containing $\vtheta_1$, $\vtheta_2$, and $\mP_{al} \vtheta_2$.}
%     \label{fig:plane_acc_additional}
% \end{figure}

\begin{figure}[htb]
    \centering
    \begin{subfigure}[b]{0.48\linewidth}
        \centering
        % \begin{subfigure}[b]{1.0\linewidth}
        % \centering
        \includegraphics[width=\linewidth]{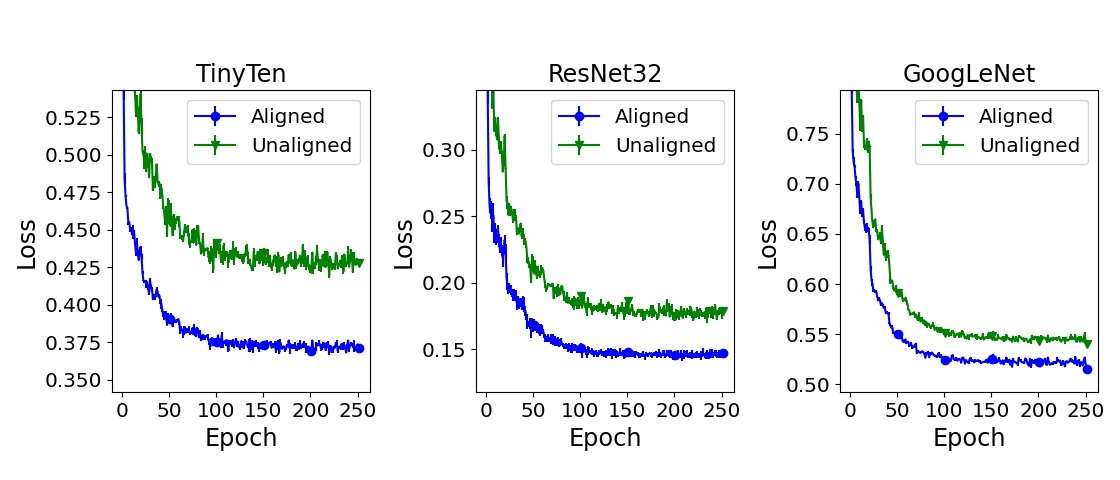}
        % \caption{Training loss}
    % \end{subfigure}
    % \hfill
    % \begin{subfigure}[b]{1.0\linewidth}
    %     \centering
    %     \includegraphics[width=\linewidth]{figures/fin_train_acc_cifar10.png}
    %     \caption{Training accuracy}
    % \end{subfigure}    
    \end{subfigure}
    \hfill
    \begin{subfigure}[b]{0.48\linewidth}
        \centering
    % \begin{subfigure}[b]{1.0\linewidth}
        % \centering
        \includegraphics[width=\linewidth]{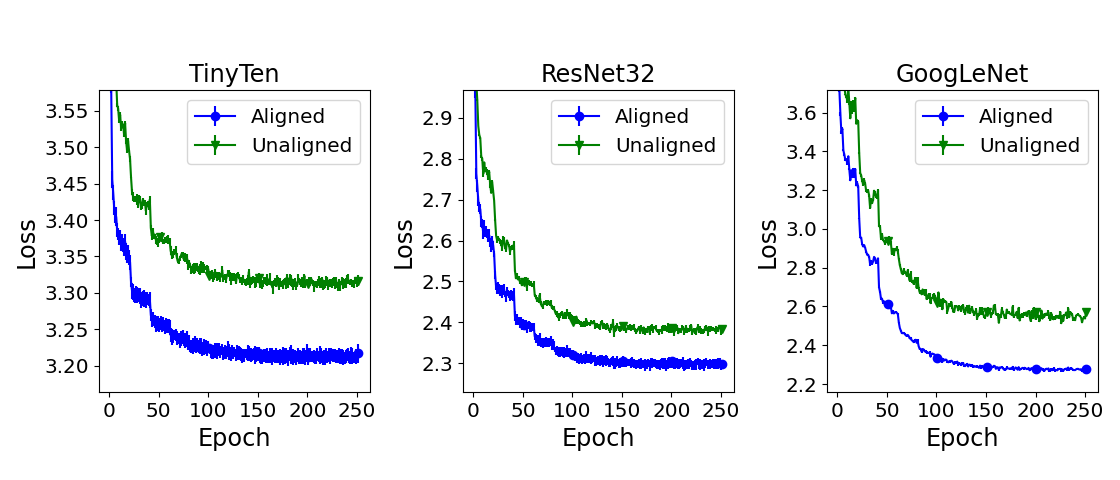}
        % \caption{Training loss}
    % \end{subfigure}
    % \hfill
    % \begin{subfigure}[b]{1.0\linewidth}
    %     \centering
    %     \includegraphics[width=\linewidth]{figures/fin_train_acc_stl10.png}
    %     \caption{Training accuracy}
    % \end{subfigure}    
    \end{subfigure}
    \vspace{-2mm}
    \caption{\textbf{Left/Right}: Training loss while learning the curve between two CIFAR10/Tiny ImageNet models.}
    \label{fig:training_cifar10}
\end{figure}

% \begin{figure}[H]
%     \centering
%     \begin{subfigure}[b]{0.95\linewidth}
%         \centering
%         \includegraphics[width=\linewidth]{figures/fin_train_loss_stl10.png}
%         \caption{Training loss}
%     \end{subfigure}
%     \hfill
%     \begin{subfigure}[b]{0.95\linewidth}
%         \centering
%         \includegraphics[width=\linewidth]{figures/fin_train_acc_stl10.png}
%         \caption{Training accuracy}
%     \end{subfigure}
%     \vspace{-2mm}
%     \caption{\textbf{Top/Bottom}: Training loss/accuracy while learning the curve between two STL10 models.}
%     \label{fig:training_stl10}
% \end{figure}

\begin{figure}[htb]
\begin{subfigure}[b]{0.48\linewidth}
    \centering
    \begin{subfigure}[b]{0.95\linewidth}
        \centering
        \includegraphics[width=\linewidth]{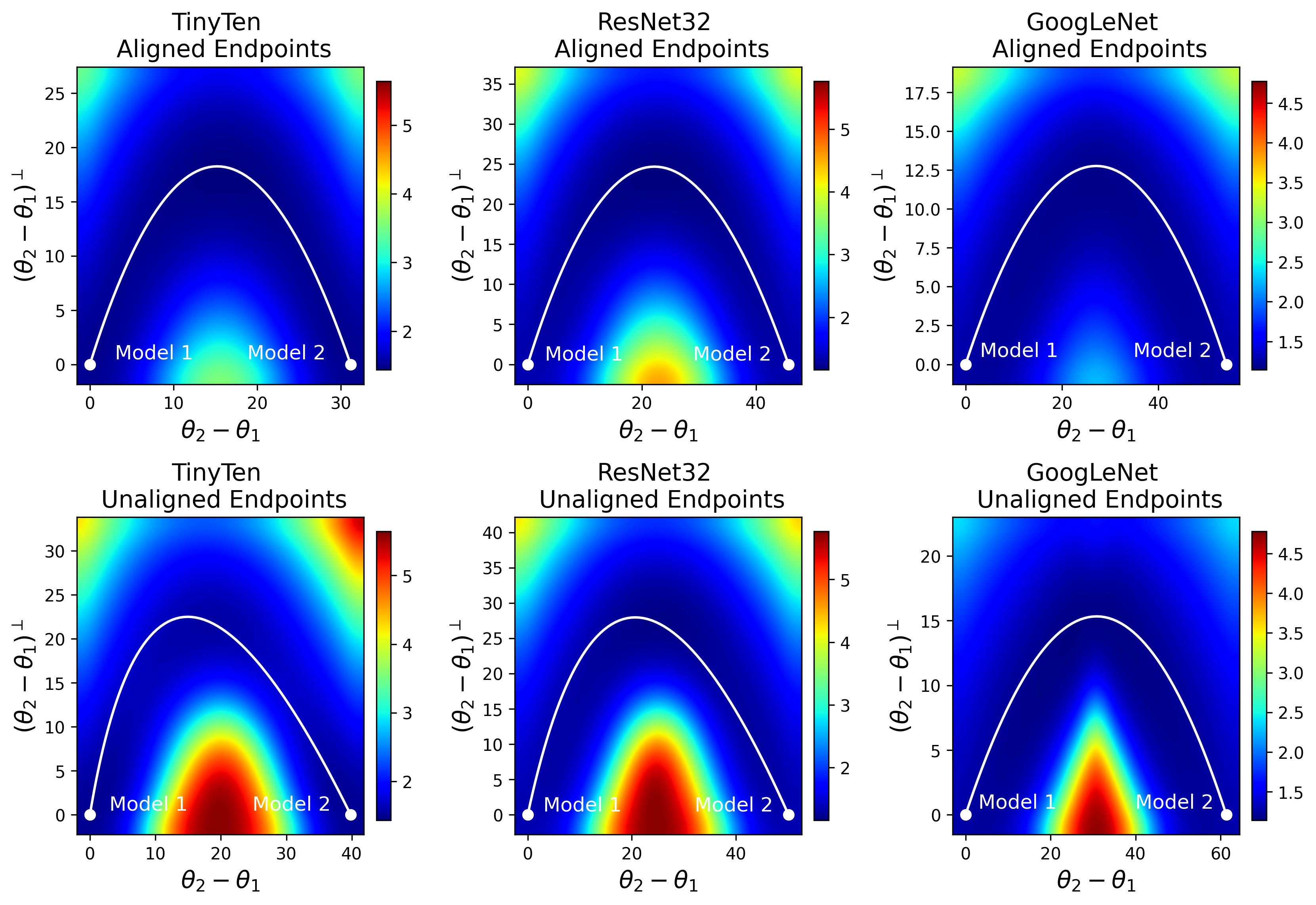}
        \caption{CIFAR100}
        \label{fig:plane_curve_loss_CIFAR100}
    \end{subfigure}
    \begin{subfigure}[b]{0.95\linewidth}
        \centering
        \includegraphics[width=\linewidth]{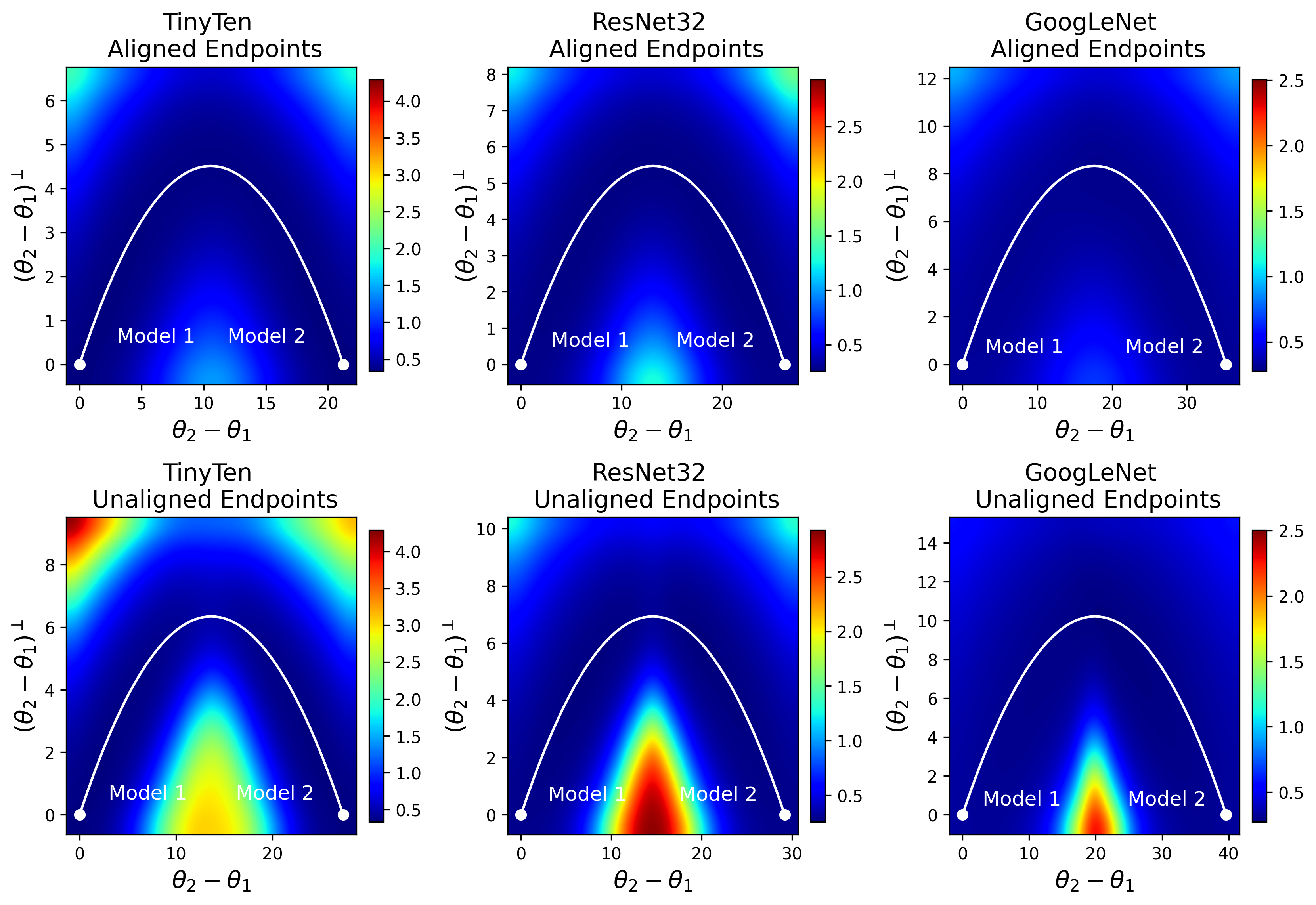}
        \caption{CIFAR10}
    \end{subfigure}
    \begin{subfigure}[b]{0.95\linewidth}
        \centering
        \includegraphics[width=\linewidth]{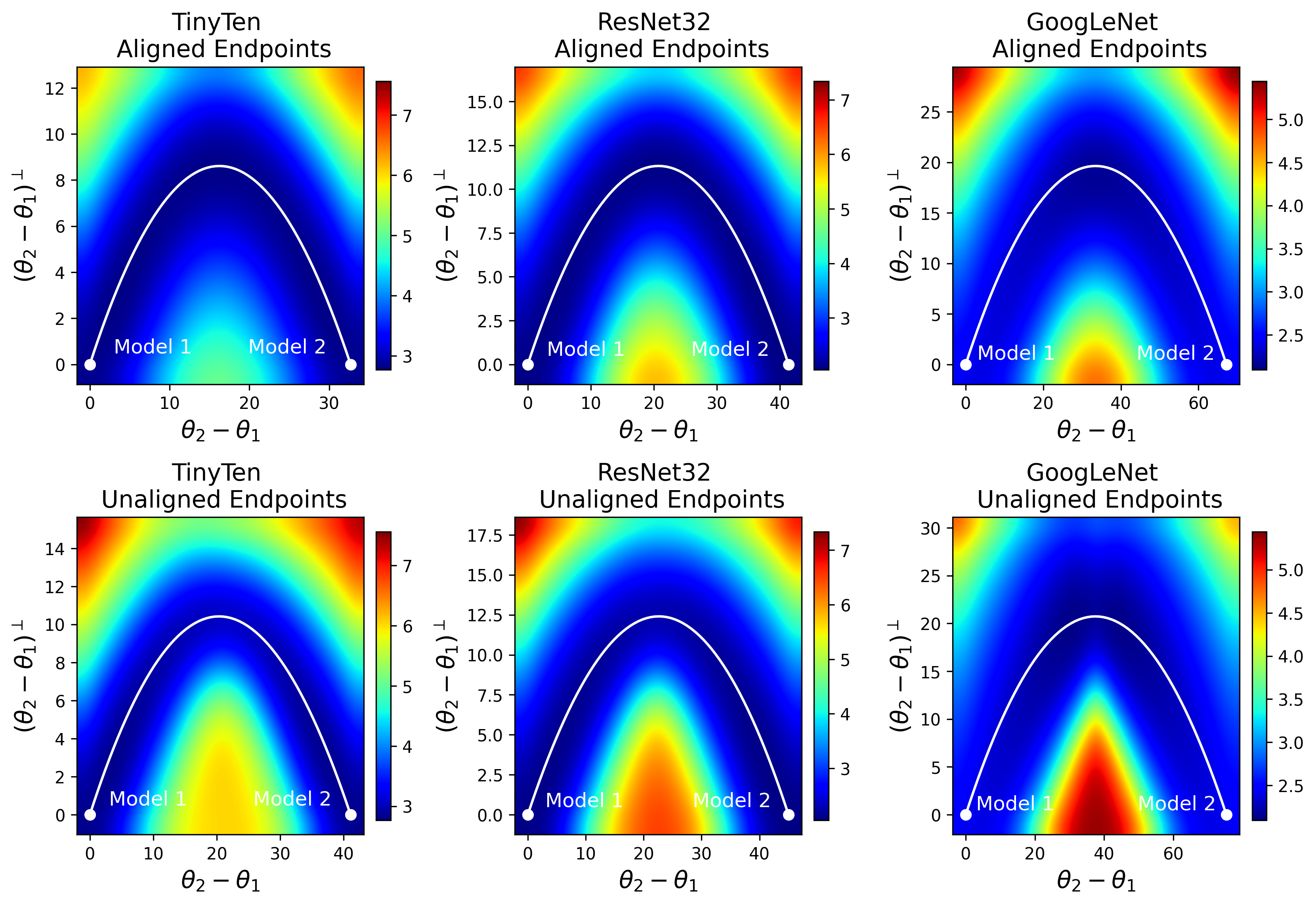}
        \caption{Tiny ImageNet}
    \end{subfigure}
\end{subfigure}
\hfill
\begin{subfigure}[b]{0.48\linewidth}
        \centering
    \begin{subfigure}[b]{0.95\linewidth}
        \centering 
            \includegraphics[width=\linewidth]{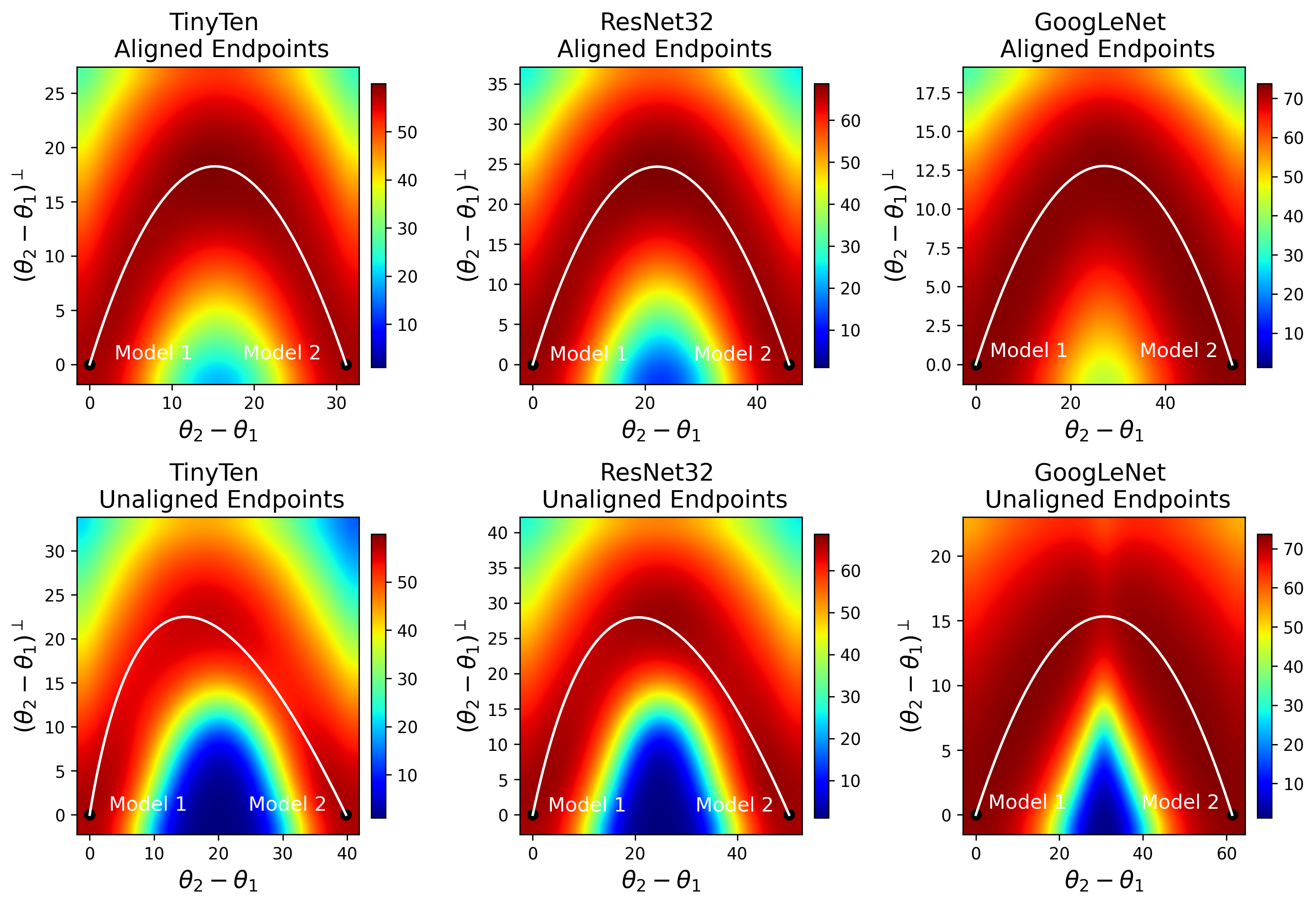} 
            \caption{CIFAR100}
    \end{subfigure}
    \begin{subfigure}[b]{0.95\linewidth}
        \centering
        \includegraphics[width=\linewidth]{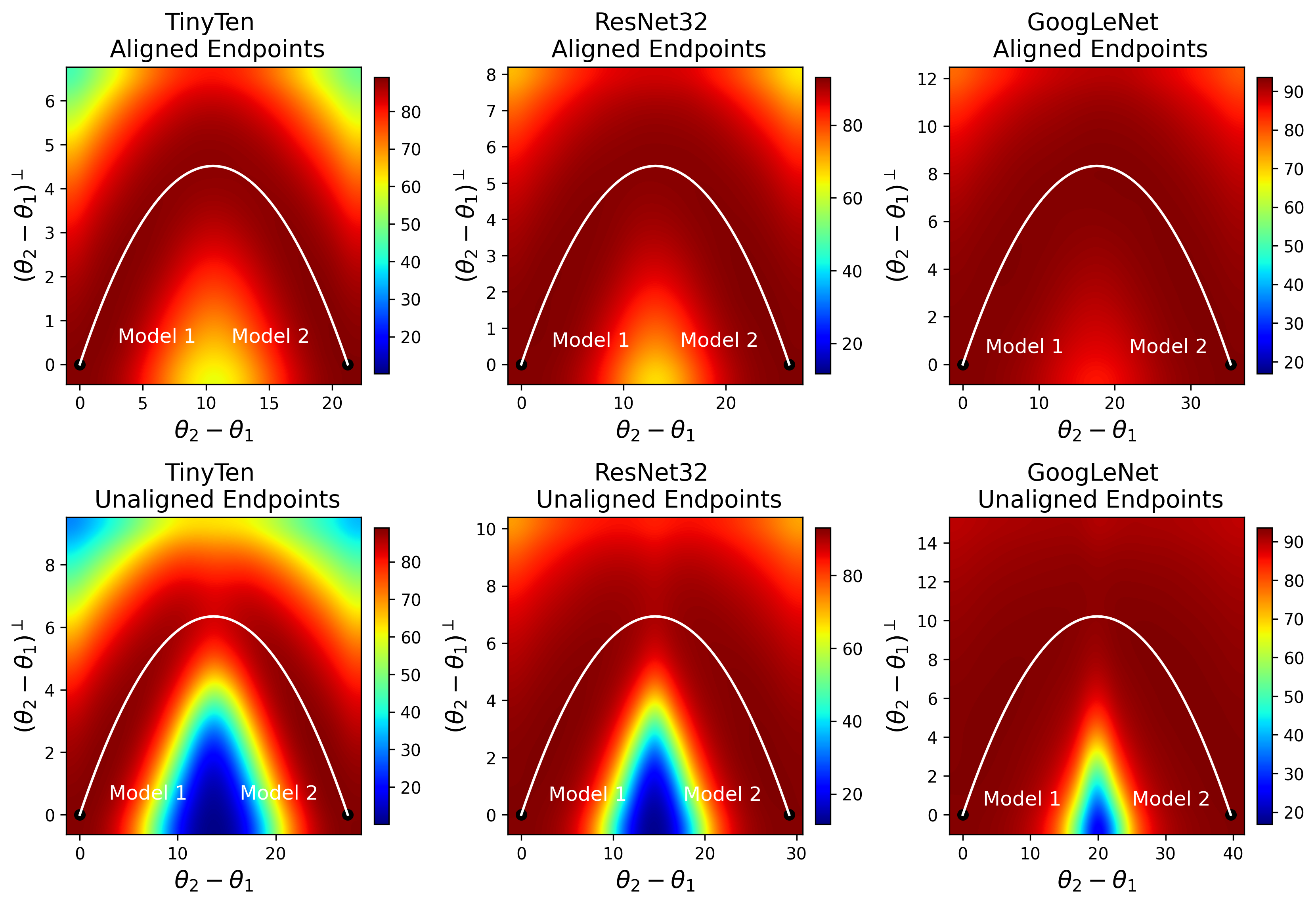}
        \caption{CIFAR10}
    \end{subfigure}
    \begin{subfigure}[b]{0.95\linewidth}
        \centering
        \includegraphics[width=\linewidth]{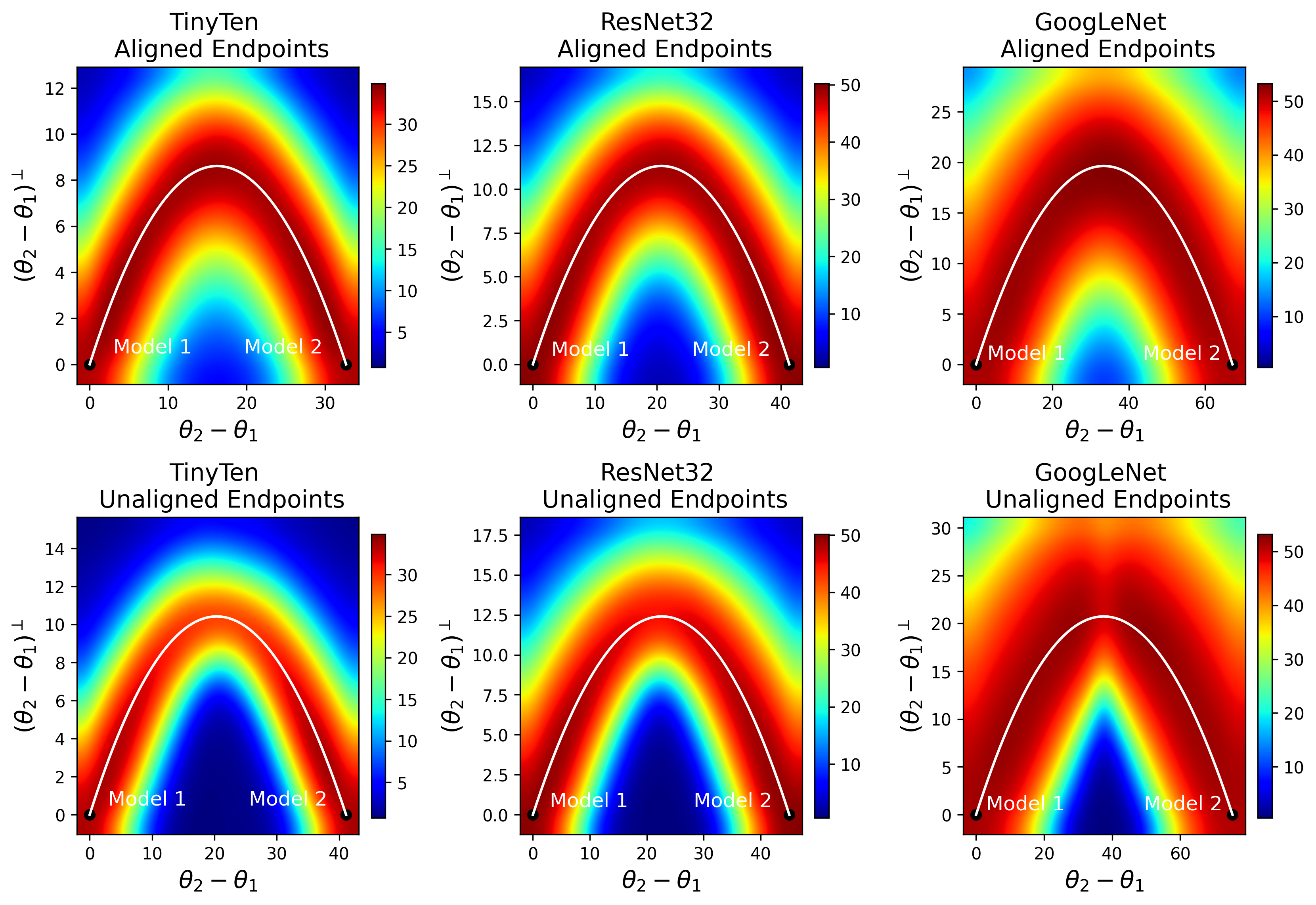}
        \caption{Tiny ImageNet}
    \end{subfigure}
\end{subfigure}
\caption{\textbf{Left/Right:} The test loss/accuracy on plane containing learned curve, $r_{\phi}(t)$.}
\label{fig:loss_acc_curve_planes_add}
\end{figure}

\begin{figure}[htb]
    \centering
    \begin{subfigure}[b]{0.48\linewidth}
        \centering
        \includegraphics[width=\linewidth]{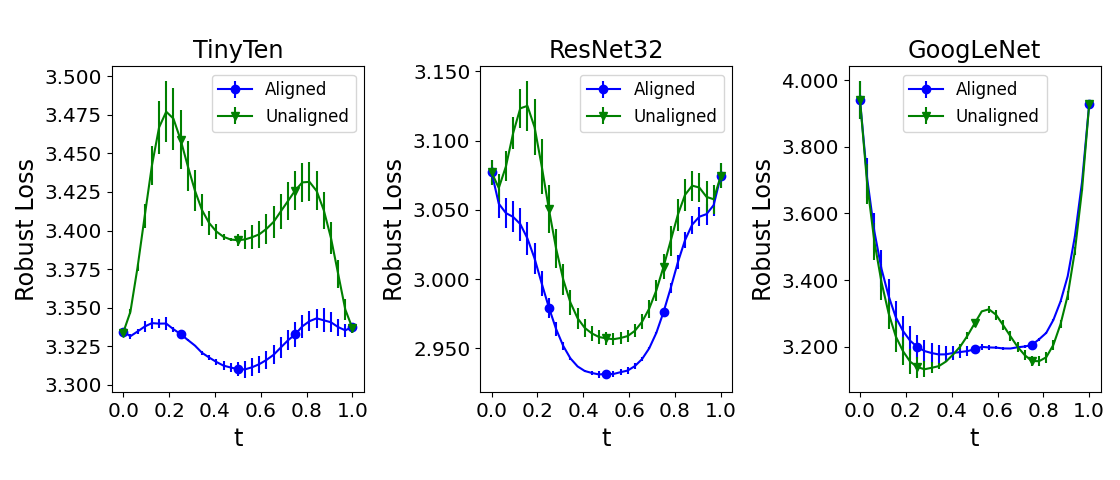}
        \caption{CIFAR100}
    \end{subfigure}
    \hfill
    \begin{subfigure}[b]{0.48\linewidth}
        \centering
        \includegraphics[width=\linewidth]{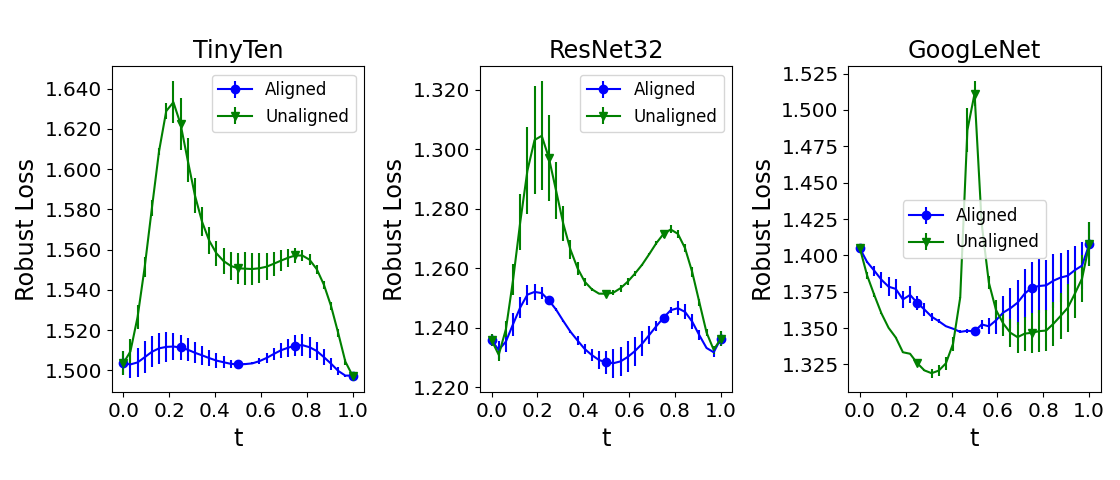}
        \caption{CIFAR10}
    \end{subfigure}    
    \caption{Robust test loss on curve between robust models.}
    \label{fig:robust_loss}
\end{figure}

\begin{figure}[htb]
    \centering
    % \begin{subfigure}[b]{0.48\linewidth}
    % \centering
        % \begin{subfigure}[b]{0.8\linewidth}
        % \centering
        % \includegraphics[width=\linewidth]{figures/cifar100_robusttrainingloss.png}
        % \caption{CIFAR100 Training Loss}
        % \end{subfigure}
        
        % \centering
        % \begin{subfigure}[b]{0.8\linewidth}
        % \centering
        % \includegraphics[width=\linewidth]{mode_connectivity/figures/cifar10_robustloss.png}
        % \caption{CIFAR10 Training Loss}
        % \end{subfigure}
    % \end{subfigure}
    % \hfill
    % \begin{subfigure}[b]{0.48\linewidth}
    \begin{subfigure}[b]{0.48\linewidth}
        \centering
        \includegraphics[width=\linewidth]{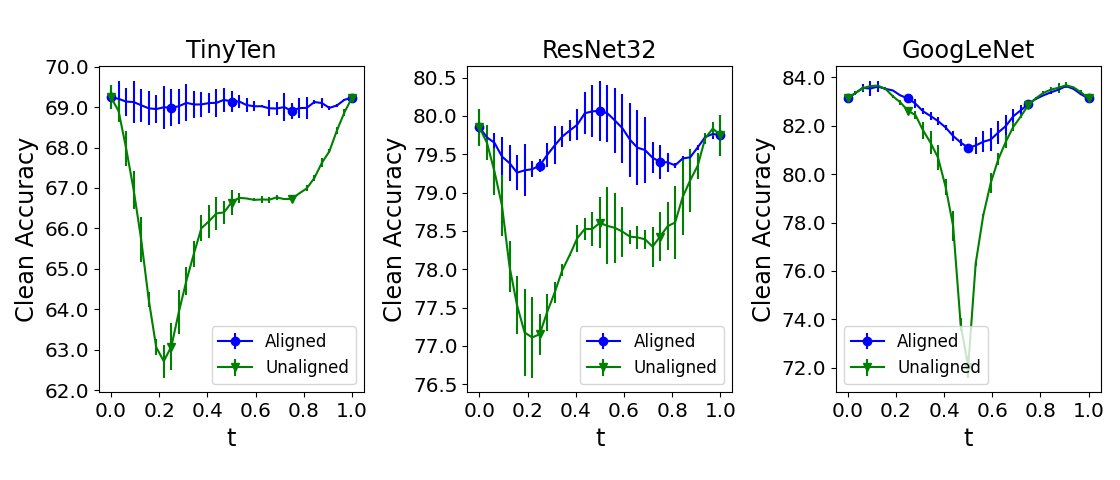}
        \caption{CIFAR10 Clean Accuracy}
    \end{subfigure}
    \begin{subfigure}[b]{0.48\linewidth}
        \centering
        \includegraphics[width=\linewidth]{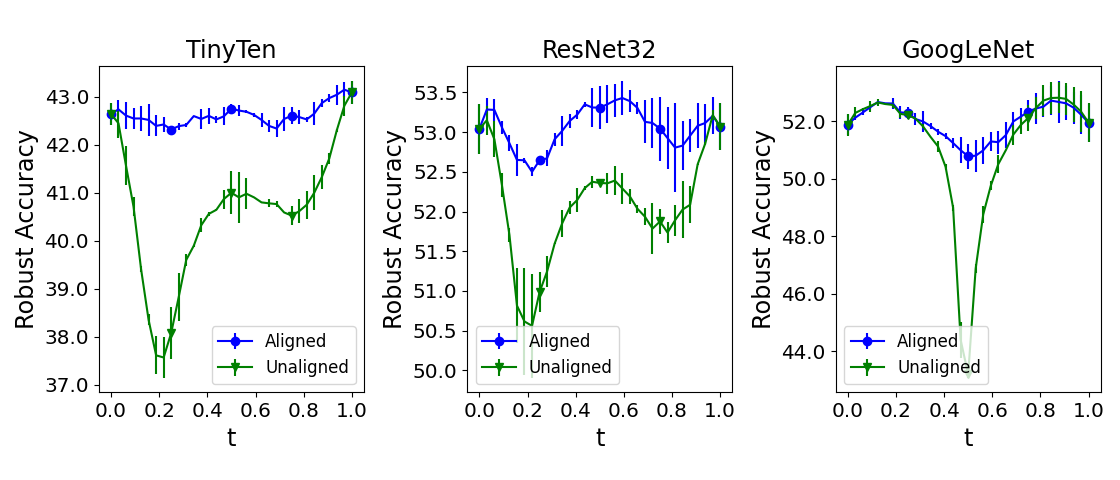}
        \caption{CIFAR10 Robust Accuracy}
    \end{subfigure} 
    % \end{subfigure}
    \caption{Clean/Robust accuracy on the CIFAR10 robust curves.}
    \label{fig:robust_cifar10}
\end{figure}

\begin{table}[htb]
\caption{The training loss with standard deviation is reported for each combination of dataset, network architecture, and curve class. GoogLeNet has higher training loss due to weight regularization.}
\label{table:results_cifar100_loss_train}
\small
% \adjustbox{max width=\textwidth}{
\centering
\begin{tabular}{ l l r r r} 
\toprule
Model & Endpoints & CIFAR10 & CIFAR100 & Tiny ImageNet\\
\midrule
\multirow{4}{4em}{TinyTen} 
& Unaligned & $0.428 \pm 0.003$ & $1.839 \pm 0.010$ & $3.317 \pm 0.008$ \\
& PAM Unaligned & $0.413 \pm 0.001$ & $1.753 \pm 0.016$ & $3.249 \pm 0.005$ \\
& PAM Aligned & $0.372 \pm 0.002$ & $\mathbf{1.679 \pm 0.005}$ & $\mathbf{3.214 \pm 0.003}$ \\
& Aligned & $\mathbf{0.371 \pm 0.002}$ & $1.693 \pm 0.008$ & $3.217 \pm 0.013$ \\
\midrule
\multirow{4}{4em}{ResNet32} 
& Unaligned & $0.179 \pm 0.001$ & $1.124 \pm 0.005$ & $2.383 \pm 0.005$ \\ 
& PAM Unaligned & $0.170 \pm 0.001$ & $1.043 \pm 0.008$ & $2.350 \pm 0.001$\\
& PAM Aligned & $\mathbf{0.147 \pm 0.002}$ & $\mathbf{0.975 \pm 0.008}$ & $2.308 \pm 0.003$ \\
& Aligned & $\mathbf{0.147 \pm 0.001}$ & $1.011 \pm 0.002$ & $\mathbf{2.299 \pm 0.009}$ \\
\midrule
\multirow{2}{5em}{GoogLeNet} & Unaligned & $0.540 \pm 0.001$ & $1.161 \pm 0.004$ & $2.570 \pm 0.009$ \\ 
& Aligned & $\mathbf{0.516 \pm 0.001}$ & $\mathbf{1.033 \pm 0.002}$ & $\mathbf{2.278 \pm 0.005}$ \\
\bottomrule
\end{tabular}
% }
\end{table}

\section{Algorithms}
\label{sec:algs}

\begin{algorithm}[htb]
 \caption{Curve Finding \citep{garipov2018loss}}
 \label{alg:curve_find}
\begin{algorithmic}
\STATE {\bfseries Input:} Two trained models, $\vtheta_1$ and $\vtheta_2$
\STATE {\bfseries Output:} A parameterized curve, $\vr_\phi$, connecting $\vtheta_1$ and $\vtheta_2$ along which loss is flat
\STATE Initialize $\vr_\phi(t)$ as $\vtheta_1 + t (\vtheta_2 - \vtheta_1)$\;
 \WHILE{not converged}
  \FOR{batch in dataset}
  \STATE sample point $\rt_0$ in $[0, 1]$
  \STATE compute loss $L(\vr_\phi(\rt_0))$ 
  \STATE optimization step on network $\vr_\phi (\rt_0)$ to update $\phi$ 
  \ENDFOR
 \ENDWHILE
\end{algorithmic}
\end{algorithm}
This section contains algorithms described in Section \ref{sec:background}. In the curve finding algorithm, the optimization step can correspond to a variety of techniques. In this paper, we use traditional stochastic gradient descent to update the curve parameters $\phi$. Notice that stochasticity is introduced by the sampling of $t$ as well as the training data. This is detailed in Algorithm \ref{alg:curve_find}.

For the purpose of computing validation loss and test loss for $\vr_\phi$, important care must be given for networks that contain batch normalization layers. This is because batch normalization aggregates running statistics of the network output that are used when evaluating the model. Though, $\vr_\phi(t_0)$ gives the weights for the model at point $t_0$, the running statistics need to be aggregated for each normalization layer. In practice, this can be done by training the model for one epoch, while freezing all learnable parameters of the model. Since batch statistics would need to be computed for each point sampled along the curve, it happens that computing the validation or test loss of the curve $\vr_\phi$ is more expensive than an epoch of training.    

\section{Theoretical Motivation for Mode Connectivity with Neuron Alignment}
\label{sec:theory_neuron_alignment}

In this section we present a theoretical discussion regarding the use of neuron alignment for curve finding up to weight symmetry. We begin by defining relevant terminology. 

\paragraph{Wasserstein distance} In the following proof, we will make use of the Wasserstein-2 metric for measuring a distance between probability distributions. This metric has recently been popular in works such as WGAN \citep{arjovsky2017wasserstein}. Formally, let $\mu$ and $\nu$ be probability measures, i.e. distributions, on a given metric space $M$. We let $\Gamma(\mu, \nu)$ denote the set of joint probability measures with marginals $\mu$ and $\nu$. Then the Wasserstein-p metric is defined as
\begin{equation}
    \label{eq:wasserstein}
    W_p(\mu, \nu) := \left(\inf_{\gamma \in \Gamma(\nu, \mu)} \E_{(\rx, \ry) \sim \gamma}[||\rx - \ry||^p] \right)^{\frac{1}{p}}
\end{equation}
Note that this metric is related to optimal transport, where $\gamma^*$ is the optimal transport plan between the two distributions with the cost being the Euclidean p-norm. Remember that the distributions we are interested in our the distributions of intermediate activations of neuron networks, $\mX_{l}^{(1)}$ and $\mX_{l}^{(2)}$, as discussed in \ref{subsec:background_align}. Then \eqref{eq:wasserstein} simplifies to 
\begin{equation}
    \label{eq:wasserstein_int_act}
    W_p(\mX_{l}^{(1)}, \mX_{l}^{(2)})^p = \min_{\mP \in \Pi_{K_l}} \sum_{i \in K_l}||\mX_{l, i}^{(1)} - \mX_{l, \mP(i)}^{(1)}||^2. 
\end{equation}
This comes from the fact that since we are dealing with an empirical distribution with uniform marginals, so the set of permutations extremal points and thus contains a minimizer. 

% \begin{figure}[htb]
% \centering
% \begin{subfigure}[b]{0.97\linewidth}
% \centering
% \includegraphics[width=\linewidth]{figures/viz_align_along_curve_unaligned_cifar100.png}
% \end{subfigure}
% \hfill
% \begin{subfigure}[b]{0.97\linewidth}
% \centering
% \includegraphics[width=\linewidth]{figures/viz_align_along_curve_aligned_cifar100.png}
% \end{subfigure}
% \vspace{-2mm}
% \caption{The mean cross-correlation between units in the curve midpoint model and each endpoint model. For context, the mean cross-correlation between the linear midpoint and each endpoint is displayed. Additionally, the mean cross-correlation between the curve midpoint and each endpoint after being aligned to the respective endpoint is displayed.}
% \label{fig:viz_align_along_curve}
% \end{figure}

% Clearly, alignment is a useful method for learning better flat loss curves between models. An interesting question is how curve finding itself relates to alignment. Until now, we have only considered the alignment between the endpoint models, $\vr(0)$ and $\vr(1)$. Now, we consider how points along the curve, $\vr(t)$, align to the endpoints. To study this numerically, we will use the curve midpoint $\vr(0.5)$. From Figure \ref{fig:loss_acc_curve_planes_add}, we see that this is the point on the quadratic Bezier curve that is roughly linearly connected to both endpoints.  

\subsection{Proof of Theorem \ref{thm:na_bound}}
\label{subsec:na_theorem}

For this proof, we consider a pair of feed-forward networks, with output defined as 
\begin{equation}
    \label{eq:one_layer_rec_net}
    \mY_j = \mW_L^{(j)} \sigma \mW_{L-1}^{(j)} \ldots \sigma \mW_1^{(j)} \mX_0, \quad i=1,2, 
\end{equation}
for the given input data distribution $X_0$. Now we consider the addition of the permutation matrices, $\mP_i$, that generalize the above equation to deal with weight symmetry,
\begin{equation}
     \mY_j = \mW_L^{(j)} \mP_{L-1}^T \sigma \mP_{L-1} \mW_{L-1}^{(j)} \mP_{L-2}^T \ldots \mP_1^T \sigma \mP_1 \mW_1^{(j)} \mX_0.
\end{equation}
The initialization used for curve finding is the interpolation between the two neural networks. This allows us to define the unaligned and aligned linear interpolations,
\begin{align}
    \vl_u(t) =& \left( (1-t)\mW_L^{(1)} + t \mW_L^{(2)} \right) \sigma \left( (1-t)\mW_{L-1}^{(1)} + t \mW_{L-1}^{(2)} \right) \ldots \\
    & \quad \sigma \left( (1-t) \mW_1^{(1)} + t \mW_1^{(2)} \right) \mX_0, \notag \\ 
    \vl_a(t) =& \left( (1-t)\mW_L^{(1)} + t \mW_L^{(2)} \mP_{L-1}^T \right) \sigma \left( (1-t)\mW_{L-1}^{(1)} + t \mP_{L-1} \mW_{L-1}^{(2)} \mP_{L-2}^T \right) \ldots \\
        & \quad \sigma \left( (1-t) \mW_1^{(1)} + t \mP_1 \mW_1^{(2)} \right) \mX_0. \notag 
\end{align}
For layer $i$ of networks along the interpolation, we define the pre-activations, $\vf_i$, and post-activations, $\vg_i$,
\begin{align}
    \vf_1^u(t) =& \left( (1-t) \mW_1^{(1)} + t \mW_1^{(2)} \right) \mX_0 \\
    \vg_i^u(t) =& \sigma \vf_i^u(t) \notag \\
    \vf_i^u(t) =& \left( (1-t)\mW_i^{(1)} + t \mW_i^{(2)} \right) \vg_{i-1}^u(t). \notag
\end{align}
These are defined similarly for the interpolation of the aligned networks, where we denote them as $\vf_i^a$ and $\vg_i^a$. 

Now we consider the $L_2$ distance between the first layer pre-activation distributions and the endpoint intermediate activation distributions. We define the following relevant $L_2$ distances:
\begin{align}
    \vd_1^u(t; 0) &= ||\vf_1^u(t) - \vf_1^u(0)||_2 = t || -\vf_1^u(0) + \vf_1^u(1) ||_2 \\
    \vd_1^u(t; 1) &= ||\vf_1^u(t) - \vf_1^u(1)||_2 = (1-t) || -\vf_1^u(0) + \vf_1^u(1) ||_2 \\
    \vd_1^a(t; 0) &= ||\vf_1^a(t) - \vf_1^u(0)||_2 = t || -\vf_1^u(0) + \mP_1 \vf_1^u(1) ||_2 \\
    \vd_1^a(t; 1) &= ||\vf_1^a(t) - \mP_1 \vf_1^u(1)||_2 = (1-t) || -\vf_1^u(0) + \mP_1 \vf_1^u(1) ||_2
\end{align}
Notice that $\mP_1$ is the permutation associated with minimizing the ground metric $L_2$ norm for the Wasserstein distance between the endpoint models $\vf_1^u(0)$ and $\vf_1^u(1)$. Then it immediately follows that 
\begin{equation}
\label{eq:l2_dist_layer1}
    \vd_1^a(t; 0) \leq \vd_1^u(t;0) \qquad \vd_1^a(t; 1) \leq \vd_1^u(t;1).
\end{equation}
Thus, we have a tighter bound on the distance between the first layer pre-activations of models along the aligned curve than the unaligned curve to those of the endpoint models. We also have that the nonlinear pointwise activation function $\sigma$ is Lipschitz continuous. Thus, there exists a constant $L_\sigma$ such that
\begin{equation}
\label{eq:lipshitz_lay}
    ||\sigma \vf_1^u(t) - \sigma \vf_1^u(0)||_2 \leq L_\sigma \vd^u_1(t; 0).
\end{equation}
Clearly, a similar relation holds for the other distances. 

We calculate our distances $\vd$ for the deeper layers of the network. We determine bounds on these distances, using $\vd$ recursively, given in the following equations:
\begin{align}
    \vd_i^u(t; 0) &= ||\vf_i^u(t) - \vf_i^u(0)||_2 \\
    & \leq L_\sigma((1-t)||\mW_i^{(1)}||_2  \vd_{i-1}^u(t;0) + t ||\mW_i^{(2)}||_2 \vd_{i-1}^u(t;1)) + t ||-\vf_i^u(0) + \vf_i^u(1)||_2 \notag \\
    \vd_i^u(t; 1) &= ||\vf_i^u(t) - \vf_i^u(1)||_2 \\
    & \leq L_\sigma((1-t)||\mW_i^{(1)}||_2 \vd_{i-1}^u(t;0) + t ||\mW_i^{(2)}||_2  \vd_{i-1}^u(t;1)) \notag \\
    & \quad + (1-t) ||-\vf_i^u(0) + \vf_i^u(1)||_2 \notag \\
    \vd_i^a(t; 0) &= ||\vf_i^a(t) - \vf_i^u(0)||_2  \\
    &\leq L_\sigma ( (1-t)||\mW_i^{(1)}||_2 \vd_{i-1}^a(t;0) + t ||\mW_i^{(2)}||_2 \vd_{i-1}^a(t;1) ) + t ||-\vf_i^u(0) + \mP_i \vf_i^u(1)||_2 \notag \\
    \vd_i^a(t; 1) &= ||\vf_i^a(t) - \mP_i \vf_i^u(1)||_2 \\
    &\leq L_\sigma ((1-t)||\mW_i^{(1)}||_2 \vd_{i-1}^a(t;0) + t ||\mW_i^{(2)}||_2 \vd_{i-1}^a(t;1))  \notag \\
    & \quad + (1-t) ||-\vf_i^u(0) + \mP_i \vf_i^u(1)||_2 \notag
\end{align}

Using \eqref{eq:l2_dist_layer1} and that $\mP_i$ is chosen to minimize the $L_2$ distance of the intermediate pre-activations of the endpoints, it follows inductively that
\begin{equation}
\label{eq:l2_dist_layeri}
    \vd_i^a(t; 0) \leq \vd_i^u(t;0) \qquad \vd_i^a(t; 1) \leq \vd_i^u(t;1).
\end{equation}
Thus, we have derived a tighter bound on the distance between the intermediate pre-activation distributions for models along the aligned linear interpolation to those of the endpoint. 

Now, we make clear that the two endpoint networks are taken to be $\epsilon$ optimal networks. That is, $||\mY - \mY_j||_2 \leq \epsilon$, where $\mY$ are the true output for the training data. Clearly, given two trained networks such an $\epsilon$ must exist. This allows the following inequalities to hold,
\begin{align}
    ||\vl_u(t) - \mY||_2 &\leq (1-t)||\mW_L^{(1)} (\sigma \vf_{L-1}^u(t) - \sigma \vf_{L-1}^u(0))||_2 \\
    & \quad + t ||\mW_{L}^{(2)} (\sigma \vf_{L-1}^u(t) - \sigma \vf_{L-1}^u(1))||_2 + \epsilon, \notag \\
    &\leq (1-t)||\mW_L^{(1)}||_2 L_\sigma \vd_{L-1}^u(t; 0) + t ||\mW_L^{(2)}||_2 L_\sigma \vd_{L-1}^u(t;1) + \epsilon. \notag
\end{align}
Similarly, we have that
\begin{align}
    \label{eq:eps_inequality}
    ||\vl_a(t) - \mY||_2 \leq (1-t)||\mW_L^{(1)}||_2 L_\sigma \vd_{L-1}^a(t; 0) + t ||\mW_L^{(2)}||_2 L_\sigma \vd_{L-1}^a(t;1) + \epsilon.
\end{align}

Finally, we have that the loss function $\mathcal{L}$ is Lipschitz-continuous or that the input dataset is bounded. If only the later case is satisfied, then it follows that the output is also bounded. As the loss function is continuous, it follows that $\mathcal{L}$ is Lipschitz-continuous restricted to the image of the dataset under the neural networks. Then there exists a constant $L_L$ such that
\begin{align}
\label{eq:local_lipschitz}
    \mathcal{L}(\vl_u(t) - \mY) &\leq L_L ||\vl_u(t) - \mY||_2 := B_u(t), \\
    \mathcal{L}(\vl_a(t) - \mY) &\leq L_L ||\vl_a(t) - \mY||_2 := B_a(t). \notag 
\end{align}
Now notice that since $\vd_{L-1}^a(t) \leq \vd_{L-1}^u(t)$, it follows that 
\begin{equation}
    B_a(t) \leq B_u(t), \quad t \in [0,1].
\end{equation}
Then we define upper bounds on the initializations for solving \eqref{eq:curve_find_sym},
\begin{equation}
    B_a := \int_0^1 B_a(t) dt \leq B_u := \int_0^1 B_u(t) dt, 
\end{equation}
where $\mP$ is fixed as unaligned or determined by alignment. Thus, they are upper bounds for the optimal solutions. This completes the proof. 

\subsection{On the Tightness of the Bounds}
\label{subsec:bound_tightness}

In the aforementioned proof, we derive tighter bounds for loss along the aligned curve compared to the unaligned curve, using the $L_2$ distance between pre-activations. We can establish the tightness of the provided bounds for a nontrivial class of networks. We do this to show that the bounds provide insight into how alignment aids mode connectivity, while having some practicality. The class of networks for which we will show tightness are networks with ReLU activation function and root mean squared error (RMSE) loss. 

We emphasize that the bound depends on the following inequalities: 
\begin{enumerate}
    \item local Lipschitz continuity of the loss function
    \item Lipschitz continuity of the activation function
    \item matrix norm inequalities for the layer weights $\mW_i$
    \item triangle inequality for expressing the intermediate activations as a linear combination of those in the previous layer
    \item triangle inequality related to $\epsilon$-optimality of the endpoints
\end{enumerate}
First we show which weights allow the bounds for the linear interpolation between models to be tight.
\begin{enumerate}
    \item Since the loss function is RMSE, \eqref{eq:local_lipschitz} achieves equality with $L_L=1$.
    \item The ReLU activation function is Lipschitz continuous with $L_\sigma = 1$. This inequality is tight for bias vector large enough so that activations are non-negative. 
    \item The matrix norm inequality is met with equality when the weights $W_i$ act as isometries on the set of activations from the previous layer. 
    \item For these triangle inequalities to be met, all terms in the sum must have the same sign. This can be accomplished by a choice of bias vectors such that $\min f_i^u(1)$ is greater than $\max f_i^u(0)$. This is in addition to $W_i^{(1)} \geq 0$ and $W_i^{(2)} \leq 0$.    
    \item The final triangle inequality can be satisfied in the following way. The signs of the weight matrices can be found such that the endpoint activations have the same sign. The bias vectors can be found such that the endpoint activations are greater than the ground truth labels for some dataset. These choices of weights define a dataset for which the endpoint models are strictly $\epsilon_1$ and $\epsilon_2$ optimal respectively. Then $\epsilon$ in \eqref{eq:eps_inequality} can be replaced with the term, $(1-t) \epsilon_1 + t \epsilon_2$. These choices guarantee the last inequality is tight, albeit with a more specified epsilon.  
\end{enumerate}
 Note that we show that there exists a choice of weights, $\epsilon_1$, and $\epsilon_2$ for which these bounds are tight.

In the main text, we discussed how tightness in the bounds for linear interpolation implies tightness in the bounds for continuous curves. Then, these bounds are nontrivial as we have tightness for a wide class of networks and curve parameterizations under a reasonable assumption. 

\subsection{On the Use of Post-Activations for Neuron Alignment}
\label{subsec:post_or_preact}

In the previous proof in Appendix \ref{subsec:na_theorem}, we assumed that the alignment is based on minimizing the $L_2$ distance of preactivations. The use of preactivation is needed in the calculation of $\vd_i$, where the term $\vf_i^u(t)$ can be linearly decomposed into $(1-t) \mW_i^{(1)} \sigma \vf_{i-1}^u(t)$ and $t \mW_i^{(2)} \sigma \vf_{i-1}^u(t)$. Such a decomposition does not necessarily hold for post-activations. 

A natural question is how Theorem \ref{thm:na_bound} can be applied to alignment of post-activations. This can be accomplished by modifying \eqref{eq:opt_trans} in the following way. Let $\mC_{l, \text{pre}}$ and $\mC_{l, \text{post}}$ be the cost matrices of the $L_2$ distances of pre-activations and post-activations in network layer $l$ respectively. Then we can define the permutation associated with aligning post-activations as 
\begin{align}\label{eq:align_preact}
    \mP_l^* = &\argmin_{\mP_l \in \Pi_{K_l}} \text{trace}(\mC_{l, \text{post}}^T \mP_l) \\
    &\text{such that} \quad  \text{trace}(\mC_{l, \text{pre}}^T \mP_l) \leq \text{trace}(\mC_{l, \text{pre}}^T). \notag
\end{align}
Given the added constraint, it follows that we can establish a tighter upper bound on the loss after aligning the post-activations, though this bound is not necessarily as tight as aligning preactivations. Using post-activations is more complicated theoretically due to the nonlinear nature of the activation function $\sigma$.  

\paragraph{On the Use of Cross-Correlation} In the main body of the paper, our numerical results concern the alignment given by maximizing the correlation of post-activations. We have just discussed theoretical details regarding the alignment of post-activations. Now, we address the use of cross-correlation. Given post-activations $\vg_i(0)$ and $\vg_i(1)$, if the distribution at each neuron is a unit normal Gaussian $\mathcal{N}(0, 1)$, then the alignment that maximizes cross-correlation is equivalent to the alignment that minimizes $L_2$ distance. In this sense, the use of cross-correlation approximates normalizing the distributions of post-activations before a $L_2$ minimizing alignment.

We provide an example to motivate the use of cross-correlation over unnormalized $L_2$ distance in our experiments. Consider a network with post-activations $\vg_i$ that will induce an alignment on the weights in the following linear layer of a neural network, $\mW_{i+1}$. With these quantities, we define quantities in what can be viewed as an equivalent network,
\begin{align}
    \hat{\vg}_i &= \text{diag}(\Sigma_{\vg_i})^{-1} (\vg_i - \E[\vg_i]), \\
    \hat{\mW}_{i+1} &= \mW_{i+1} (\mI + \E[\vg_i]) \text{diag}(\Sigma_{\vg_i}) \notag .  
\end{align}
Note that these are pointwise normalizations of the post-activations. Then it is reasonable that if we were to align the activations $\vg_i$ and $\hat{\vg}_i$, we would want an alignment invariant to affine transformations that can essentially be absorbed into the following linear layer. Additionally, maximizing cross-correlation as opposed to minimizing the $L_2$ distance of normalized distributions leads to the easy-to-interpret correlation signature seen in Figure \ref{fig:corr_align}.

In Figure \ref{fig:align_comp}, we compare these different techniques for alignment. Empirically, we find that aligning post-activations outperforms aligning pre-activations. Additionally, alignment via maximizing cross-correlation is seen to outperform all other methods. This validates are decision to use this technique in the main body of the paper.

\section{Residual Network Alignment}
\label{sec:resnet_align}

Algorithm \ref{alg:alignment} applies to networks with a typical feed-forward structure. In this section, we discuss how we compute alignments for the ResNet32 architecture as it is more complicated. It is important to align networks such that the network structure is preserved and network activations are not altered. In the context of residual networks, special consideration must be given to skip connections.

Consider the formulation of a basic skip connection,
\begin{equation} 
    \mX_{k+1} = \sigma \circ ( \mW_{k+1} \mX_k ) + \mX_{k-1} 
\end{equation}
In this equation, we can see that $\mX_{k+1}$ and $\mX_{k-1}$ share the same indexing of their units. This becomes clear when you consider permuting the hidden units in $\mX_{k-1}$ without permuting the hidden units of $\mX_{k+1}$. It is impossible to do so without breaking the structure of the equation above, where there is essentially the use of an identity mapping from $\mX_{k-1}$ to $\mX_{k+1}$. This effect that skip connections has on the symmetries of the loss surface has been studied previously in \citep{orhan2017skip}. We note that the skip connection does not eliminate this symmetry, the symmetry is now just shared across certain layers. 

We consider a traditional residual network that is decomposed into residual blocks. In each block the even layers have skip connections while the odd layers do not. So, we compute the alignment as usual for odd layers. For all even layers within a given residual block, we determine a shared alignment. We do this by solving the assignment problem for the average of the cross-correlation matrix over the even layers in that residual block.  

\section{Proximal Alternating Minimization for Solving the Joint Model}
\label{sec:PAM}

We introduce a framework to solve the generalized problem in \eqref{eq:curve_find_sym}. Theoretically, this problem is fairly complicated and hard to analyze. Numerically, approaching the problem directly with first order methods could be computationally intensive as we need to store gradients of $\phi$ and $\mP$ simultaneously. The problem can be more easily addressed using the method of proximal alternating minimization (PAM) \citep{attouch2010proximal}. The PAM scheme involves iteratively solving the two subproblems in \eqref{eq:PAM}. Here we let $Q(\phi, \mP)$ denote the objective function in \eqref{eq:curve_find_sym}. We only consider parameterized forms of $\vr$ that satisfy the endpoint constraints for all $\phi$ and $\mP$. 
\begin{equation}\label{eq:PAM}
\begin{cases}
    \mP^{k+1} =  \argmin\limits_{\mP} & Q(\phi^{k}, \mP) + \frac{1}{2 \nu_{P}} ||\mP - \mP^k||_2^2 \\
    \qquad \text{such that} & \text{blockdiag}(\mP_1, \mP_2, ..., \mP_{L-1}) \quad \text{where} \quad  
    \mP_l \in \Pi_{|K_l|} \\
    \phi^{k+1} =  \argmin\limits_{\phi} &  Q(\phi, \mP^{k+1}) + \frac{1}{2 \nu_{\phi}} ||\phi - \phi^k||_2^2
\end{cases}
\end{equation}
Computing the unaligned curve is equivalent to solving the PAM scheme with a very small value of $\nu_P$. In fact, we are able to prove local convergence results for a certain class of networks.

\setcounter{topnumber}{0}

\begin{theorem}[Convergence]\label{thm1}
Let $\{\phi^{k+1}, \mP^{k+1}\}$ be the sequence produced by \eqref{eq:PAM}. Assume that $\vr_{\phi}(t)$ corresponds to a feed-forward neural network with activation function $\sigma$ for $t \in [0, 1]$. Assume that $\mathcal{L}$, $r_{\phi}$, and $\sigma$ are all piece-wise analytic functions in $C^1$ and locally Lipschitz differentiable in $\phi$ and $\mP$. Lastly, assume that the input data is bounded and the norm of the network weights are constrained to be bounded above. Then the following statements hold:
\begin{enumerate}
    \item $Q(\phi^{k+1}, \mP^{k+1}) + \frac{1}{2 \nu_{\phi}} ||\phi^{k+1} - \phi^k||_2^2 + \frac{1}{2 \nu_{P}} ||\mP^{k+1} - \mP^k||_2^2 \leq Q(\phi^{k}, \mP^{k}), \quad \forall k \geq 0$ %\item $\sum_{k=1}^\infty \left( ||\phi^{k} - \phi^{k-1}||^2 + ||\mP^{k} - \mP^{k-1}||^2 \right) < \infty$
    \item $\{\phi^k, \mP^k\}$ converges to critical point of \eqref{eq:curve_find_sym}.
\end{enumerate}
\end{theorem}
% \begin{proof}
% See Appendix \ref{sec:proofs}. 
% \end{proof}
\textit{Proof.} \quad See Appendix \ref{sec:proofs}. 

\setcounter{topnumber}{2}

% \paragraph{Remark}
% Theorem \ref{thm1} does not extend to neural networks with ReLU activation functions. In Appendix \ref{sec:proofs}, we address a technique utilizing this theorem for learning a curve connecting rectified networks while still generating a sequence of iterates with monotonic decreasing objective value. 

% However, given two rectified networks $\vtheta_1$ and $\vtheta_2$, we can learn the curve between them by substituting huberized forms of ReLU within the PAM scheme during curve training. Under some assumptions, we can establish an approximation error associated with huberization, and guarantee that our sequence is still monotonic decreasing with respect to the original network architecture. Details of this can be found in Appendix \ref{sec:proofs}. 
\subsection{Quadratic Bezier Curve Parameterization}
We explicitly define the quadratic Bezier curve for use in the PAM algorithm in \eqref{eq:bezier_perm}. Here the curve has been reparameterized so that the control point is a function of the permutation $\mP$. $\tilde{\vtheta_c}$ captures the deviation of the control point from the linear midpoint between $\vtheta_1$ and $\mP \vtheta_2$. For PAM, $\tilde{\vtheta_c}$ is the learnable curve parameter in $\phi$. It is zero initialized so that the initial curve is a linear interpolation between models as in traditional curve finding. This coupling of the control point with the permutation is critical for the success of PAM. 
\begin{equation} \label{eq:bezier_perm} 
\begin{split}
    \vr(t; \tilde{\vtheta_c}, \mP) & = (1 - t)^2 \vtheta_1 + t^2 \mP \vtheta_2  + 2 (1 - t) t \left(\frac{\vtheta_1 + \mP \vtheta_2}{2}  + \tilde{\vtheta_{c}} \right). 
\end{split}
\end{equation}

\subsection{Numerical Implementation for PAM}
\label{PAM_results}

% Proximal alternating minimization provides a comprehensive formulation for learning the weight permutation $\mP$ directly, coupled with some convergence guarantees. We find that curves learned using PAM perform better than the unaligned curves as seen in Table \ref{table:results_cifar10}. As was the case for the aligned curves, this performance gain is more notable in underparameterized models. Notably, the aligned curves perform comparably to PAM aligned. This indicates that $\mP$ is already close to the locally optimal permutation when $\mP$ is chosen as the initialization for PAM. Additionally, the performance gain of PAM Aligned over PAM Unaligned shows that this permutation is not easy to learn when $\mP^{(0)}$ is not necessarily close. Then training aligned curves is an inexpensive way to approximate the solution to a rigorous optimization method with good initialization. 
% We stress that this observation shows that the gain from neuron alignment is not trivial.    

% Additionally, if we assume that the most optimal curve will be learned from the most optimal linear path, then the loss of the PAM Aligned case is a lower bound on the loss of the Aligned case. This could be subject to some numerical error associated with combinatorial optimization techniques.  

To learn each PAM curve, we perform a single outer iteration of PAM. This was seen as sufficient for training to converge. The permutation subproblem entails 20 epochs of projected stochastic gradient descent to the set of doubly stochastic matrices. This is done as the set of doubly stochastic matrices is the convex relaxation of the set of permutations.  This projection is accomplished through 20 iterations of alternating projection of the updated permutation to the set of nonnegative matrices and the set of matrices with row and column sum of $1$. After the 20 epochs of PGD, each layer permutation is projected to the set of permutations, $\Pi_{|K_l|}$. This projection is detailed in Appendix \ref{sec:permutation_sampling}. The curve parameter subproblem, which optimizes $\tilde{\vtheta_c}$ from \eqref{eq:bezier_perm}, entails 250 epochs of SGD. The same hyperparameters are used as in training the other curves. The learning rates are annealed with each iteration of PAM.    

% The outlier case in Table \ref{table:results_cifar10} is for ResNet32 where Aligned notably outperforms PAM Aligned. This result is due to the integral nature of the problem. Essentially, during PAM Aligned, the scheme learns a more optimal doubly stochastic matrix that actually projects to a worse permutation at the end of the permutation subproblem. This is a known possible negative, when solving a problem in the convex hull and then projecting back to the feasible set.

\subsection{Proofs for PAM}
\label{sec:proofs}

For the following proofs, we first establish and more rigorously define some terminology. We first discuss an important abuse of notation. For clarity the parameterized curve connecting networks under some permutation $\mP$ that has been written as $\vr_\phi(t)$ will now sometimes be referred to as $\vr(t; \phi, \mP)$. 

\paragraph{Feed-forward neural networks} In this section, we will be analyzing feed-forward neural networks. We let $\mX_0 \in \R^{m_0 \times d}$ be the input to the neural network, $d$ samples of dimension $m_0$. Then we let $\mW_i \in \R^{m_i \times m_{i-1}}$ denote the network weights mapping from layer $l-1$ to layer $l$. Additionally, $\sigma$ denotes the pointwise activation function. Then we can express the output of a feed-forward neural network, $\mY$, as:
\begin{equation}
    \mY := \mW_{L} \sigma \circ \mW_{L-1} \sigma \circ \mW_{L-2} ... \sigma \circ \mW_{1} \mX_0
\end{equation}
To include biases, $\{\vb_i\}_{i=1}^{L}$, we simply convert to homogeneous coordinates, 
\begin{align}
    \hat{\mX_0} = \begin{bmatrix}
    \mX_0 \\
    1 
    \end{bmatrix},
    \quad 
    \hat{\mW_i} = \begin{bmatrix}
    \mW_i & \vb_i \\
    \mathbf{0} & 1 
    \end{bmatrix},
    \quad 
    \hat{\mY} = \begin{bmatrix}
    \mY \\ 1
    \end{bmatrix}
\end{align}
In all proofs, these terms are interchangeable. 

\paragraph{Huberized ReLU} The commonly used ReLU function is defined as $\sigma(t) := \max(0, t)$. However, this function is not in $C^1$ and hence not locally Lipschitz differentiable. This makes conducting analysis with this function difficult. Thus, we will approach studying it through the lens of the huberized ReLU function, defined as:
\begin{equation}
    \sigma_\delta(t) := \begin{cases}
    0 & \text{ for } t \leq 0 \\
    \frac{1}{2 \delta} t^2 & \text{ for } 0 \leq t \leq \delta \\
    t - \frac{\delta}{2} & \text{ for } \delta \leq t
    \end{cases}
\end{equation}
It is clear that $\sigma_\delta$ is a $C^1$ approximation of $\sigma$ such that $||\sigma - \sigma_\delta||_\infty = \frac{\delta}{2}$. Using huberized forms of loss functions for analysis is a fairly common technique such as in \citep{xu2016proximal} which studies huberized support vector machines.

\paragraph{Kurdyka-Lojasiewicz property} The function $f$ is said to have the Kurdyka-Lojasiewics (KL) property at $\bar{x}$ if there exist $\nu \in (0, +\infty]$, a neighborhood $U$ of $\bar{x}$ and a continuous concave function $\psi: [0, \nu) \rightarrow \mathbb{R}_{+}$ such that:
\begin{itemize}
    \item $\psi(0) = 0$ 
    \item $\psi$ is $C^1$ on $(0, \nu)$
    \item $\forall s \in (0, \nu)$, $\psi'(s) > 0$
    \item $\forall x \in U \cap [f(\bar{x}) < f < f(\bar{x}) + \nu]$, the Kurdyka-Lojasiewics inequality holds
    \begin{equation}
        \psi'(f(x) - f(\bar{x})) \text{dist}(0, \partial f(x)) \geq 1.
    \end{equation}
\end{itemize}
Here $\partial f$ denotes the subdifferential of $f$. Informally, a function that satisfies this inequality is one whose range can be reparameterized such that a kink occurs at its minimum. More intuitively, if $\psi$ has the form, $s^{1-\theta}$ with $\theta$ in $(0,1)$, and $f$ is differentiable on $(0, \nu)$, then the inequality reduces to 
\begin{equation}
    \frac{1}{(1 - \theta)} |f(x) - f(\bar{x})|^\theta \leq ||\nabla f(x)||
\end{equation}

\paragraph{Semialgebraic function} A subset of $\mathbb{R}^n$ is semialgebraic if it can be written as a finite union of sets of the form
\begin{equation*}
    \{x \in \mathbb{R}^n: p_i(x) = 0, q_i(x) < 0, i = \{1, 2, ..., p\}\}
\end{equation*}
where $p_i$ and $q_i$ are real polynomial functions. A function $f: \mathbb{R}^n \rightarrow \mathbb{R} \cup \{+ \infty\}$ is said to be semialgebraic if its graph is a semialgebraic subset of $\mathbb{R}^{n+1}$.

\paragraph{Subanalytic function} Globally subanalytic sets are sets that can be obtained through finite intersections and finite unions of sets of the form $\{(x, t) \in [-1, 1]^n \times \mathbb{R}: f(x) = t\}$ where $f: [-1, 1]^n \rightarrow \mathbb{R}$ is an analytic function that can be extended analytically on a neighborhood of the interval $[-1, 1]^n$. A function is subanalytic if its graph is a globally subanalytic set.

\subsubsection{Proof of Theorem \ref{thm1}}

To prove this, we need that our problem meets the conditions required for local convergence of proximal alternating minimization (PAM) described in \citep{attouch2010proximal}. This requires the following:
\begin{enumerate}
    \item Each term in the objective function containing only one primal variable is bounded below and lower semicontinuous. 
    \item Each term in the objective function which contains both variables is in $C^1$ and is locally Lipschitz differentiable.
    \item The objective function satisfies the Kurdyka-Lojasiewicz (KL) property.
\end{enumerate}

First we reformulate the problem so that it becomes unconstrained. Let $\chi$ denote the indicator function, where:
\begin{equation}\label{eq:indicator}
    \chi_C(t) := \begin{cases}
    0, \quad & \text{for } t \in C \\
    + \infty, \quad & \text{otherwise} 
    \end{cases}
\end{equation}
This problem contains two hard constraints. First, each permutation matrix, $P_l$, must clearly be restricted to the set of permutation matrices of size $|K_l|$, $\Pi_{|K_l|}$. Additionally, it is assumed that the norm of the weights are bounded above. Without loss of generality, let $K_W$ denote an upper bound valid for all the weights. We denote the set of weights that satisfy the norm constraint as $\{A: ||A||_2^2 \leq K_W\}$. Then \eqref{eq:curve_find_sym} with added regularization is equivalent to:
\begin{equation} \label{eq:curve_find_sym_unconstrained}
\begin{split}
    \phi^*, \mP^* = \arg \min_{\phi, \mP}  \quad & Q(\phi, \mP) +  \sum_{l=1}^{L-1} \chi_{\Pi_{|K_l|}}(\mP_l) 
    + \sum_{l=1}^{L} \chi_{\{A: ||A||_2^2 < K_W\}}(W_l) 
\end{split}
\end{equation}
We now address each requirement for local convergence. 
\begin{enumerate}
    \item From \eqref{eq:indicator}, we can see that the sum of indicator functions are bounded below and lower semicontinuous.
    \item Now we consider the form of the function, $Q(\phi, \mP)$. It has been defined as 
    \begin{equation*}
        \int_{t=0}^1 \mathcal{L}(\vr(t; \phi, \mP)) dt 
    \end{equation*}
    We know that $\vr(t; \phi, \mP)$ corresponds to a feed-forward neural network. Then $Q$ can be expressed as:
    \begin{equation}\label{eq:expand_proof}
    \begin{split}
        \int_{t=0}^1 & \mathcal{L} \left( W_L(t; \phi, \mP) \sigma \circ W_{L-1}(t; \phi, \mP) \ldots \sigma \circ W_1(t; \phi, \mP) X_0 \right) dt
    \end{split}
    \end{equation}
    with weight matrices $W_i$ and activation function $\sigma$. It becomes clear that for $Q(\phi, \mP)$ to be in $C^1$ and locally Lipschitz differentiable, the same must be true for $\mathcal{L}$, $\sigma$, and $\{W_i\}_{i=1}^L$. The first two are true as they are assumptions of the theorem. Since, $r_\phi$ is in $C^1$ and locally Lipschitz differentiable in the primal variables, then this is also true for all $W_i$. Thus, $Q(\phi, \mP)$ is in $C^1$ and locally Lipschitz differentiable.  
    
    \item To satisfy the KL property, the objective function must be a \textit{tame} function \citep{attouch2010proximal}. Rigorously, this means that the graph of the function belongs to an o-minimal structure, a concept from algebraic geometry. We refer curious readers to \citep{van2002minimal} for further reference. 
    
    First, we note that $Q(\phi, \mP)$ is piece-wise analytic. This is because $Q$ is a composition of piece-wise analytic functions, $\mathcal{L}$, $\sigma$, and $r_\phi$. Additionally, because the input data is bounded and the norm of the weight matrices are bounded, it follows that the domain of $Q$ is bounded. Since, $Q$ is a piece-wise analytic function with bounded domain, it follows that $Q$ is a subanalytic function. The boundedness of the domain is an important detail here. This is because analytic functions are not necessarilly subanalytic unless their domain is bounded; a popular example of such a function is the exponential function. 
    % The regularization function, $\mathcal{R}$, is assumed to be a piece-wise analytic function. It follows from the previous reasoning that $\mathcal{R}$ has bounded domain. Thus, $\mathcal{R}$ is a subanalytic function. 
    
    We now consider the constraints associated with this problem, which have been re-expressed as indicator functions in the objective. The set of permutation matrices, $\Pi_{|K_l|}$, is finite and thus it is clearly a semi-algebraic set. Notice that the set of weight matrices satisfying the norm bound is equivalent to $\{A: ||A||_2^2 - K_W < 0\}$. The function that defines this set is a polynomial, so it is a semi-algebraic set. Indicator functions on semi-algebraic sets are semi-algebraic functions. Thus, the indicator functions in the objective are semi-algebraic. 
    
    The graphs of semi-algebraic functions and subanalytic functions both belong to the logarithmic-exponential structure, an o-minimal structure. A basic algebraic property of o-minimal structures is that the graphs of addition and multiplication are also elements of the structure \citep{van2002minimal}. Since our objective function is a linear combination of semi-algebraic functions and subanalytic functions, it follows that the graph of our objective function is an element of the logarithmic-exponential structure. Therefore, our objective function is a \textit{tame} function and it satisfies the KL property. 
\end{enumerate}

\subsubsection{Considering Rectified Networks}

Theorem \ref{thm1} does not extend to the class of rectified networks. However, we are still interested in contructing a sequence of iterates $\{\phi^k, \mP^k\}$ such that the objective value, $\E_{t \sim U}[\mathcal{L}(\vr(t; \phi^k, \mP^k))]$, is monotonic decreasing. The following theorem will introduce a technique for constructing such a sequence. 

\begin{lemma}[$\mathcal{L}$ restricted to possible network outputs is Lipschitz continuous]\label{lem:lipschitz}
For a feed-forward neural network, assume that $\mathcal{L}$ is continuous and that the neural network input, $X_0$, is bounded. Additionally, assume that the spectral norm of all weights, $\{W_i\}_{i=1}^L$, is bounded above by $K_W$, and the activation function, $\sigma$, is continous with $||\sigma|| \leq 1$. Let $S_Y$ denote the set of $Y$ where 
\begin{align}
    Y &= \mW_{L} \sigma \circ \mW_{L-1} \sigma \circ \mW_{L-2} ... \sigma \circ \mW_{1} \mX_0 \\
    & \text{such that} \quad ||\mW_i||_{2} \leq K_W \quad \forall i \in \{1, 2, ..., L\} \notag 
\end{align}
Then $\mathcal{L}$ restricted to the set $S_Y$ is Lipschitz continuous with some Lipschitz constant $K$.
\end{lemma}
\begin{proof}
Since $X_0$ is bounded, it follows that there exists some constant $K_X$ such that $||X_0|| \leq K_X$. Since, the spectral norm of $W_1$ is bounded above by $K_W$, it is easy to see that $||W_1 X_0|| \leq K_W K_X$. Now since the pointwise activation function is a non-expansive map, it immediately follows that $||\sigma \circ W_1 X_0|| \leq K_W K_X$. Following this process inductively, we see that the network output, $Y$, is bounded and that: 
\begin{equation}
    ||Y|| \leq K_W^{L} K_X
\end{equation}
Since $Y$ is arbitrary, it follows that this is a bound for $S_Y$. Then we can restrict $\mathcal{L}$ to the ball in $\R^{m_L \times d}$ of radius $K_W^L K_X$. This ball is compact and $\mathcal{L}$ is continuous, so it follows that $\mathcal{L}$ restricted to this ball is Lipschitz continuous. Thus, there exists some Lipschitz constant $K$. Clearly, $S_Y$ is contained in this ball. Therefore, $\mathcal{L}$ is Lipschitz continuous on the set of all possible network outputs with Lipschitz constant $K$. 
\end{proof}

Let $\theta_1$ and $\theta_2$ be feed-forward neural networks with ReLU activation function. Assume that $\mathcal{L}$ and $r_\phi$ are piece-wise analytic functions in $C^1$ and locally Lipschitz differentiable. Assume that the maximum network width at any layer is $M$ units. Additionally, assume that the network weights have a spectral norm bounded above by $K_W$, and that this is a hard constraint when training the networks. Finally, any point on $r_\phi$ must be equivalent to an affine combination of neural networks (Bezier curves, polygonal chains, etc.) satisfying the previously stated spectral norm bound.   

Create the parameterized curve $\vr_{\delta}(t; \phi, \mP)$ by substituting the huberized ReLU function, $\sigma_\delta$, into all ReLU functions in $\vr(t; \phi, \mP)$. We refer to the objective values associated with these curves as $Q_\delta(\phi, \mP)$ and $Q(\phi, \mP)$ respectively. 

\begin{theorem}[Monotonic Decreasing Sequence for Rectified Networks]
For a feed-forward network, assume the above assumptions have been met. Additionally, assume that $X_0$ is bounded, so that $\mathcal{L}$ restricted to the set of possible network outputs is Lipschitz continuous with Lipschitz constant $K_L$ by Lemma \ref{lem:lipschitz}. Now generate the sequence $\{\phi^k, \mP^k\}$ by solving \eqref{eq:PAM} for $r_\delta(t; \phi, \mP)$. On this sequence impose the additional stopping criteria that
\begin{equation}
\begin{split}
    \frac{1}{2 \nu_{\phi}} ||\phi^{k+1} - \phi^k||_2^2 + \frac{1}{2 \nu_{P}} ||\mP^{k+1} - \mP^k||_2^2 
    \geq K_L \sqrt{M} \frac{\delta}{2} \sum_{i=1}^{L-1} K_W^i \qquad \forall k \geq 0 .
\end{split}
\end{equation}
Then, the sequence of curves $\vr(t; \phi^k, \mP^k)$ connecting rectified networks has monotonic decreasing objective value.   
\end{theorem}

\begin{proof}
First we consider the approximation error from replacing $\sigma$ with $\sigma_\delta$. It is straightforward to see that \begin{equation}
    \max_t |\sigma(t) - \sigma_\delta(t)| \leq \frac{\delta}{2}. 
\end{equation}
Then it follows that for any input $\vx$,
$$||\sigma \circ W_1 \vx - \sigma_\delta \circ W_1 \vx||_2 \leq \sqrt{M} \frac{\delta}{2}.$$
Since the spectral norm of $W_i$ are bounded above by $K_W$, then we see that
$$||W_2 \sigma \circ W_1 \vx - W_2 \sigma_\delta \circ W_1 \vx||_2 \leq K_W \sqrt{M}\frac{\delta}{2}.$$
Now notice that
\begin{equation}
\begin{split}
     ||\sigma \circ W_2 \sigma \circ W_1 \vx - \sigma_\delta \circ W_2 \sigma_\delta \circ W_1 \vx||  & \leq 
    ||\sigma \circ W_2 \sigma \circ W_1 \vx - \sigma \circ W_2 \sigma_\delta \circ W_1 \vx|| \\ & \qquad + ||\sigma \circ W_2 \sigma_\delta \circ W_1 \vx - \sigma_\delta \circ W_2 \sigma_\delta \circ W_1 \vx|| . \notag 
\end{split}
\end{equation}
Since the ReLU function is a non-expansive map, it must be that the first term is bounded above by the previous error, $K_W \sqrt{M} \frac{\delta}{2}$. The second term corresponds once again to the error associated with the huberized form of the ReLU function, $\sqrt{M} \frac{\delta}{2}$. Thus the total error can be bounded by $(K_W + 1) \sqrt{M} \frac{\delta}{2}$.

Following this inductively, it can be seen that the this error grows geometrically with the number of layers. Additionally, the loss function is Lipschitz continuous when restricted to the set of possible network outputs. So we find the following bounds:
\begin{align}
    ||Y - Y_\delta|| &\leq \sqrt{M} \frac{\delta}{2} \sum_{i=1}^{L-1}K_W^i \notag \\
    ||\mathcal{L}(Y) - \mathcal{L}(Y_\delta)|| &\leq K_L \sqrt{M} \frac{\delta}{2} \sum_{i=1}^{L-1}K_W^i \label{eq:bound_thm}    
\end{align}
Since any point on the curve is an affine combination of networks with the $K_W$ bound on the spectral norm of their weights, it immediately follows this spectral norm bound also holds for the weights for any point on the curve. Then $||Q(\phi, \mP) - Q_\delta(\phi, \mP)||$ is also bounded above by the bound in \eqref{eq:bound_thm}. 

Then let $\{\phi^k, \mP^k\}$ be the sequence generated by solving \eqref{eq:PAM} using the curve $r_\delta$. $\sigma_\delta$ is a piece-wise analytic function in $C^1$ and is locally Lipschitz differentiable. Additionally, the spectral norm constraint on the weights is semi-algebraic and bounded below, so Theorem \ref{thm1} can be applied. It then follows that
\begin{equation}
\begin{split}
& Q(\phi^{k+1}, \mP^{k+1}) + \frac{1}{2 \nu_{\phi}} ||\phi^{k+1} - \phi^k||_2^2  + \frac{1}{2 \nu_{P}} ||\mP^{k+1} - \mP^k||_2^2 \\
& \qquad \leq Q(\phi^{k}, \mP^{k}) + K_L \sqrt{M} \frac{\delta}{2} \sum_{i=1}^{L-1}K_W^i, \quad \forall k \geq 0
\end{split}
\end{equation}
Thus, $\vr(t; \phi^k, \mP^k)$ is a sequence of curves, connecting rectified networks, with monotonic decreasing objective value as long as 
\begin{equation*}
\begin{split}
    & \frac{1}{2 \nu_{\phi}} ||\phi^{k+1} - \phi^k||_2^2 + \frac{1}{2 \nu_{P}} ||\mP^{k+1} - \mP^k||_2^2 
    \geq K_L \sqrt{M} \frac{\delta}{2} \sum_{i=1}^{L-1}K_W^i \qquad \forall k \geq 0
\end{split}
\end{equation*}
Since the above equation is a stopping criterion introduced in the theorem statement, it follows that we have constructed a sequence of curves, connecting rectified networks, with monotonic decreasing objective value. 
\end{proof}

% \clearpage

\subsection{Details on Solving the Permutation Subproblem in PAM}
\label{sec:permutation_sampling}

In this section, we provide additional details on the solution to the permutation subproblem in \eqref{eq:PAM}. For short-hand notation, we re-express the subproblem as
\begin{equation}
\label{eq:perm_subproblem}
    \mP^{k+1} = \argmin_{\mP \in \text{blockdiag}(\mP_1, \ldots, \mP_{L-1}); s.t. \mP_i \in \Pi_{|K_i|}} L(\mP; \phi^k, \mP^k)
\end{equation}
As described in Section \ref{PAM_results}, when considering the permutation subproblem, we begin by solving the convex relaxation of the subproblem. To be clear, we solve for the locally optimal matrix $\mD^*$, which is blockwise doubly-stochastic, minimizing $L$ in \eqref{eq:perm_subproblem} , using projected stochastic gradient descent.

Critical to solving the permutation subproblem is obtaining $\mP^{k+1}$ given $\mD^*$. We utilize $\mD^*$ to determine a set of block permutation matrices, $S$, with $\mP^{k+1}$ as an element. Specifically, 
\begin{equation}
    \mP^{k+1} := \argmin_{\mP \in S} \mathcal{L}(\mP; \phi^k, \mP^k).
\end{equation}
% Here $S$ is all possible block permutation matrices given the determined candidate permutations for each block, $S_i$. 

Intuitively, a natural candidate for $S$ is the projection of $\mD^*$ to the set of permutations, $\Pi$. We refer to this permutation as $P_{\Pi}(\mD^*)$. This solution is a classic heuristic for solving integer programs. This heuristic is precisely solving the convex relaxation of an integer program and then projecting the optimal solution back to the feasible set. In practice, this projection can be solved using the Hungarian algorithm. That is, the projection is the solution to
\begin{equation}
    \mP_{\Pi}(\mD^*) := \argmin_{\mP \in \Pi} - \text{trace}(\mP^T \mD^*).
\end{equation}

As this projection is a heuristic, there is no guarantee that $\mP_{\Pi}(\mD^*)$ is more optimal than the previous permutation, $\mP^{k}$. To this end, we also let $\mP^{k}$ be in $S$. This prevents us from having our loss increase after solving the subproblem. 

It is also possible that there exists a permutation near $\mD^*$ that is more optimal than $\mP_{\Pi}(\mD^*)$. To this end, we are interested in randomly sampling such matrices. To do this, we will construct the block permutation matrix $\mR := \text{blockdiag}(\mR_1, \ldots, \mR_{L-1})$ where $\mR_i$ is the permutation matrix corresponding to the \textit{i}th layer and is sampled from a distribution $\Omega_i$. For more concise notation, we will say that $R \sim \Omega$. As there is no definitive way for sampling permutation matrices from doubly stochastic matrices, we detail how we construct the distributions $\{\Omega_l\}_{l=1}^{L-1}$ as follows.  

A well-known result is that every doubly stochastic matrix is a convex combination of permutation matrices, with this convex combination being known as a Birkhoff-von Neumann (BvN) decomposition of the matrix \citep{dufosse2016notes}. That is, for any doubly stochastic matrix $D$, there exists a set of permutation matrices such that
\begin{equation}
    \label{bvn_equation}
    \mD = \sum_{i \in I} \alpha_i \mP_i \quad \text{s.t.} \quad \alpha \geq 0, \sum_{i \in I} \alpha_i = 1, \mP_i \in \Pi.
\end{equation}
The problem of determining the minimal size of the index set $I$ is known to be an NP hard problem, and a doubly stochastic matrix can have multiple BvN decompositions. For sampling the permutation from layer $l$, $\mP_l$, we utilize a BvN decomposition. This is because the BvN decomposition lends itself to having a probabilistic interpretation. Given the decomposition in \eqref{bvn_equation}, we view the permutation $\mP_i$ being sampled from $\Omega_l$ with probability $\alpha_i$. 

To determine the specific BvN decomposition that we use, we introduce a variant of the \textit{greedy Birkhoff heuristic} \citep{dufosse2016notes}. First, let $\mD^{(1)} := \mD$. We associate a bipartite graph, $G^{(1)}$, with the matrix. In this bipartite graph, the two vertex sets correspond to the rows and columns of $\mD^{(1)}$ with an edge indicating a nonzero value in the corresponding entry of $\mD^{(1)}$. We then consider the set of permutation matrices corresponding to a perfect matching in $G$, $\Pi_M$. With this, we define the first permutation in the BvN decomposition as 
\begin{equation}\label{eq:our_bvn_heuristic}
    \mP_1 = \argmax_{\mP \in \Pi_M} \text{trace}(\mP^T \mD^{(1)}).
\end{equation}
Then $\alpha_1$ is taken to be the smallest nonzero term in $\mP^T \mD^{(1)}$. Following this, we take $\mD^{(2)} = \mD^{(1)} - \alpha_1 \mP_1$ and iteratively repeat this process until the full BvN decomposition has been constructed. 

This BvN decomposition can easily be solved iteratively using the Hungarian algorithm. Note that $P_{\Pi}(\mD_l^*)$ is guaranteed to be the first term of the BvN decomposition with the highest value of $\alpha_1$ possible among BvN decompositions including that permutation. This motivates the use of this heuristic, as it can be seen as a natural probabilistic extension of direct projection to the permutation set. The traditional \textit{greedy Birkhoff heuristic} determines $\mP_l$ in the following way instead, 
\begin{equation}
    \mP_l = \argmax_{\mP \in \Pi_M} \min \text{diag}(\mP^T \mD^{(1)}).
\end{equation}
In practice, we truncate $\alpha$ after $i=10$, and normalize the truncated coefficient vector to determine the distribution $\Omega_l$. We sample $M$ block permutation matrices $\mR$ from $\Omega$, with $M=32$. This is clearly a small sample given the dimension of $P$, but it is large enough to provide matrices that outperform $P_{\Pi}(\mD^*)$ in our experiments.

Then all together, we have that the set $S$ contains the block permutation matrices, $\{P^{k}, P_{\Pi}(\mD^*), \{\mR^{(t)}\}_{t=1}^M; \mR^{(t)} \sim \Omega \}$. 

\section{Hyperparameter Search}
\label{sec:hyperparam_search}

\begin{table}[tb]
  \caption{Average test accuracy of models along the learned curves trained with different hyperparameters. The curves connect the TinyTen architecture and are trained on the CIFAR100 dataset. The accuracy of any choice of aligned curve exceeds that of all unaligned curves. This establishes that the performance gains associated with alignment are not sensitive to choice of batch size or intial learning rate.}
  \label{tab:hyperparam_study}
  \adjustbox{max width=\textwidth}{
  \centering
  \begin{tabular}{@{}rrrrcrrr@{}}
    \toprule
    & \multicolumn{3}{c}{Unaligned curves} & \phantom{a} & \multicolumn{3}{c}{Aligned curves} \\
    Learning rate & \multicolumn{3}{c}{Batch size} && \multicolumn{3}{c}{Batch size} \\
    \midrule
    & 64 & 128 & 256 && 64 & 128 & 256 \\
    \cmidrule{2-4} \cmidrule{6-8}
    1E-2 & $55.9 \pm 0.2$ & $55.3 \pm 0.3$ & $54.2 \pm 0.3$ && $58.4 \pm 0.2$ & $57.9 \pm 0.3$ & $57.4 \pm 0.1$ \\
    1E-1 & $55.9 \pm 0.1$ & $56.0 \pm 0.2$  & $56.0 \pm 0.1$ && $59.0 \pm 0.1$ & $58.7 \pm 0.2$ & $58.7 \pm 0.1$\\
    5E-1 & $55.5 \pm 0.3$ & $55.7 \pm 0.3$ & $55.5 \pm 0.1$ && $58.6 \pm 0.2$ & $58.7 \pm 0.1$ & $58.9 \pm 0.1$  \\
    \bottomrule
  \end{tabular}
  }
\end{table}

In the main paper, the CIFAR100 curves are trained with an initial learning rate of 1E-1 and a batch size of 128. The choice of batch size is the same as in \citep{garipov2018loss} while training was seen to converge with the given initial learning rate. Still, it is important to establish that are main results are not dependent on the choice of hyperparameters. 

In Table \ref{tab:hyperparam_study}, the results for training curves for the TinyTen architecture and CIFAR100 dataset using three different choices of initial learning rate and batch size are displayed. We see that the curves are mostly optimal when the choice of batch size is not too large and the learning rate is not too small. Regardless, it is clear that the accuracies associated with the aligned curves exceed those of the unaligned curves. Namely, the lowest of the aligned accuracies is notably greater than the highest of the unaligned accuracies. Thus, we can conclude that alignment improves mode connectivity while not being sensitive to the choice of hyperparameters. We believe this insensitivity to hyperparameters will extend to different architectures and datasets.

\end{document}